\documentclass[
]{ceurart}

\usepackage{amsmath}
\usepackage{graphicx}
\usepackage{microtype}
\usepackage{dsfont}
\usepackage{stmaryrd}
\usepackage{tikz} 
\usepackage{comment}
\usetikzlibrary{shapes.misc,automata,positioning,arrows}
\usepgflibrary{shapes.arrows} 
\usepackage{xcolor}
\usepackage{comment}
\usepackage{mathrsfs}
\usepackage{latexsym}
\usepackage{subcaption}

\usepackage{multicol}

\usepackage{listings}
\usepackage{amsthm}

\newtheorem{lemma}{Lemma}
\newtheorem{definition}{Definition}
\newtheorem{example}{Example}
\newtheorem{proposition}{Proposition}
\newtheorem{corollary}{Corollary}

\begin{document}


\newcommand{\ie}{\mbox{i.e.}}
\newcommand{\eg}{\mbox{e.g.}}
\newcommand{\cf}{\mbox{cf.}}
\newcommand{\viz}{\mbox{viz.}}
\newcommand{\wrt}{\mbox{w.r.t.}}
\newcommand{\st}{\mbox{s.t.}}
\newcommand{\wolog}{\mbox{w.l.o.g.}}
\newcommand{\etal}{\mbox{et al}}
\newcommand{\aka}{\mbox{a.k.a.}}
\newcommand{\myskip}{\smallskip} 
\newcommand{\df}[1]{\textbf{#1}} 


\newcommand{\Lang}{\mathcal{L}}
\newcommand{\KB}{\mathcal{K}}


\newcommand{\dentails}{\mid\hskip-0.40ex\approx}   
\newcommand{\ndentails}{\not\mid\hskip-0.40ex\approx}

\newcommand{\ptwiddle}{\mathrel|\joinrel\sim}        
\newcommand{\nptwiddle}{\mathrel|\joinrel\not\sim}   
\newcommand{\twiddle}{\leadsto}        
\newcommand{\ntwiddle}{\not\leadsto}   

\newcommand*{\myleftmid}{%
	\mathrel{\vcenter{\offinterlineskip
			\vskip-0.25ex\hbox{$\shortmid$}}}}
\newcommand*{\myrightmid}{%
	\mathrel{\vcenter{\offinterlineskip
			\vskip-0.7ex\hbox{$\shortmid$}}}}
\newcommand*{\twosim}{%
	\mathrel{\vcenter{\offinterlineskip
			\vskip0.05ex\hbox{$\sim$}\vskip0.25ex\hbox{$\sim$}}}}
\newcommand{\bartwosim}{\mathrel{\myleftmid}\hskip-0.03ex\joinrel\twosim}
\newcommand{\dnec}{\mathrel{\bartwosim}\hskip-.05ex\joinrel\myrightmid}
\newcommand{\dposs}{\scalebox{0.8}{\raisebox{-0.2ex}{\rotatebox{57}{\ensuremath{\dnec}}}}\hskip-0.3ex}

\newcommand{\Lflag}{\Lang^{\scalebox{0.5}{${\dnec}$}}}
\newcommand{\Langd}{\widetilde{\Lang}}
\newcommand{\LangdSt}{\Langd_{\mathds{S}}}
\newcommand{\LangProp}{\Lang^{\ptwiddle}}

\newcommand{\dimp}{\leadsto}
\newcommand{\ndimp}{\not\leadsto}

\newcommand{\SSS}{\mathcal{S}}
\newcommand{\T}{\mathcal{T}}
\newcommand{\B}{\mathcal{B}}

\newcommand{\states}[1]{\llbracket #1 \rrbracket}

\newcommand{\Nick}[1]{\textcolor{orange}{#1}}
\newcommand{\Ivan}[1]{\textcolor{blue}{#1}}
\newcommand{\Tommie}[1]{\textcolor{red}{#1}}


\newcommand{\animal}{\ensuremath{\mathsf{animal}}}
\newcommand{\cheese}{\ensuremath{\mathsf{cheese}}}
\newcommand{\egg}{\ensuremath{\mathsf{egg}}}
\newcommand{\Env}{\ensuremath{\mathsf{Environmentalist}}}
\newcommand{\Pcf}{\ensuremath{\mathsf{Pacifist}}}
\newcommand{\Vga}{\ensuremath{\mathsf{Vegan}}}
\newcommand{\Vgt}{\ensuremath{\mathsf{Vegetarian}}}
\newcommand{\animalF}{\ensuremath{\mathsf{a}}}
\newcommand{\cheeseF}{\ensuremath{\mathsf{c}}}
\newcommand{\eggF}{\ensuremath{\mathsf{e}}}
\newcommand{\EnvF}{\ensuremath{\mathsf{Env}}}
\newcommand{\PcfF}{\ensuremath{\mathsf{Pcf}}}
\newcommand{\VgaF}{\ensuremath{\mathsf{Vga}}}
\newcommand{\VgtF}{\ensuremath{\mathsf{Vgt}}}

\newcommand{\Phy}{\ensuremath{\mathsf{Phy}}}
\newcommand{\PhyF}{\ensuremath{\mathsf{P}}}
\newcommand{\Ast}{\ensuremath{\mathsf{Astro}}}
\newcommand{\AstF}{\ensuremath{\mathsf{A}}}
\newcommand{\Eng}{\ensuremath{\mathsf{Eng}}}
\newcommand{\EngF}{\ensuremath{\mathsf{E}}}

\newcommand{\NL}{\ensuremath{\mathsf{NL}}}
\newcommand{\VG}{\ensuremath{\mathsf{VG}}}
\newcommand{\inertia}{\ensuremath{\mathsf{i}}}
\newcommand{\NLF}{\ensuremath{\mathsf{n}}}
\newcommand{\VGF}{\ensuremath{\mathsf{v}}}

\newcommand{\I}{\mathcal{I}}
\newcommand{\R}{\mathcal{R}}

\newcommand{\naf}{\mathrm{not }}

\copyrightyear{2026}
\copyrightclause{Copyright for this paper by its authors.
  Use permitted under Creative Commons License Attribution 4.0
  International (CC BY 4.0).}

\conference{ }

\title{Towards Non-Monotonic Entailment in Propositional Defeasible Standpoint Logic}

\author[1]{Nicholas Leisegang}
\cormark[1]
\address[1]{University of Cape Town and CAIR, South Africa}
\author[1]{Thomas Meyer}
\author[2,1,3]{Ivan Varzinczak}
\address[2]{Université Sorbonne Paris Nord, Inserm, Sorbonne Université, Limics, 93017 Bobigny, France}
\address[3]{ISTI-CNR, Pisa, Italy}

\cortext[1]{Corresponding author.}

\begin{abstract}
Recent work in defeasible reasoning has seen notions of preferential semantics and entailment in the style of Kraus et al. applied to modal logics. However, work in this field has focussed primarily on satisfiability checking, and monotonic notions of entailment, which may be inferentially weak.  One particular modal logic where this has been introduced is propositional standpoint logics, where modalities can express the views of different viewpoints. This has resulted in the  formalisation of propositional defeasible standpoint logic (PDSL). In this paper, we propose a means of lifting the class of (non-monotonic) rational entailment relations from traditional KLM-style reasoning to a fragment of PDSL. In order to do so, we extend the expressivity of PDSL via situated standpoint conditionals, allowing us to talk about a defeasible conditional holding in the context of a given standpoint. This allows us to re-characterise the syntax of PDSL in terms of situated conditionals, and shows that a large fragment of PDSL is expressible as a set of situated conditionals. We then focus on characterising non-monotonic entailment in this fragment, defining a method to transport any ranking-based entailment relation from the propositional case into the PDSL case. This is first described in the general case and then considered in the specific cases of rational and lexicographic closures, providing a faithful translation of each inference into PDSL. We also show that entailment-checking in this fragment of PDSL can be done largely using algorithms from the propositional case, while preserving complexity bounds.
\end{abstract}


\maketitle

\medskip


\section{Introduction}

Standpoint modalities were introduced in the propositional case by Gómez Álvarez and Rudolph \cite{alvarezrudolph:propositionalstdpt}, in order to integrate (potentially conflicting) viewpoints into a single logical framework while maintaining consistency. However, this framework could only talk about the unequivocal and possible beliefs which a standpoint holds. Leisegang et al. \cite{JELIA-NickIvanTommie} introduced preferential semantics to classical standpoint logics which added the ability to reason about the beliefs which are \textit{typically held} by or \textit{distinctly possible} to a standpoint. In their work, they consider a Tarskian notion of entailment which reasons monotonically over all models. This however, often becomes unnessecarily inferentially weak when reasoning over defeasible statements, as can be seen in the following example.

\begin{example}\label{example:introductory-example}
    Within the standpoint of physics, Newton's laws are usually taken to hold. This can be represented within defeasible standpoint logics as $\dnec_{\Phy}\NL$, which states that \textit{``the physics standpoint typically believes in Newton's laws (NL)''}. One of Newton's laws is the law of universal gravitation, which states that gravitational force acting between two objects is a function of the mass of the objects and the distance between them. Thus, Newton's laws usually treat gravitational force as variable. This can expressed using the defeasible implication $\NL\twiddle \VG$ where VG stands for ``variable gravity''. Taking these two pieces of information into account, it seems reasonable to conclude that from the physics perspective, gravitational force should be typically treated as variable. That is, we should conclude $\dnec_{\Phy}\VG$. Moreover, assume we are introduced to a different standpoint, such as that of astronomy or engineering, which bases its views on ideas from physics. We can express this in PDSL by $\Ast\lesssim\Phy$ and $\Eng\lesssim\Phy$, stating that typically the astronomy and engineering standpoints can be seen as more specific cases of the physics standpoint. Based on this, it is reasonable that these sub-standpoints inherit the views of the greater standpoint and we conclude that $\dnec_\Ast \NL$ and $\dnec_\Eng \NL$. However, from the perspective of a (typical) engineer who works on Earth-based projects, it is practical to treat gravity not as a variable, but as a constant value that the Earth exerts on materials. Hence, we may want to add that although $\Eng\lesssim\Phy$ holds, we have that $\dnec_\Eng \neg \VG$. That is, from an engineering perspective, gravitational force is usually not considered a variable. 
\end{example}

Leisegang et al.~\cite{JELIA-NickIvanTommie} introduce Propositional Defeasible Standpoint Logic (PDSL) as a means to increase the expressivity of propositional standpoint logics via the introduction of defeasible standpoint modalities, defeasible implications and defeasible standpoint sharpenings. The above example shows a situation where the increased expressivity can allow the logic to incorporate increased nuance through defeasibility. For example, the classical knowledge base $\{\Box_{\Phy}\VG, \Eng\leq\Phy, \Box_{\Eng} \neg \VG\}$ is unsatisfiable, while its defeasible variant in PDSL $\{\dnec_{\Phy} \VG, \Eng\lesssim\Phy, \dnec_{\Eng} \neg \VG\}$ can be satisfied. In previous work \cite{JELIA-NickIvanTommie}, entailment for PDSL was restricted to the Tarskian notion of preferential entailment, denoted $\vDash_P$. However, this notion of entailment in PDSL severely restricts the inferences which follow from a PDSL knowledge base. For instance, we have that,
\[\{{\dnec_{\Phy}}\NL, \NL\twiddle \VG\}\nvDash_P {\dnec_{\Phy}}\VG\quad\textit{ and }\quad\{{\dnec_{\Phy}\NL},\Ast\lesssim\Phy\}\nvDash_P{\dnec_{\Ast}}\NL\]
which is shown in detail in Example \ref{example:SPSS-that-shows-p-entailment-is-weak}. Hence, our ability to reason prototypically about standpoints falls at the first hurdle when considering $\vDash_P$. In the first case, we are not able to extract any of the usual conclusions which our standpoints ought to make from their beliefs even though there is no known reason to prevent this. In the second, sub-standpoints are unable to inherit any of the typical believes of their greater standpoint, even though we are given no reason to believe they behave atypically as a sub-standpoint. In order to combat this, we aim to lift methods of non-monotonic defeasible reasoning from the traditional KLM literature into the PDSL case. Specifically, we consider the class of rational defeasible inference operations, including rational closure and lexicographic closure. The paper is organised as follows. In Section \ref{section:prelims} we give preliminaries on PDSL and traditional (propositional) KLM defeasible reasoning. In Section \ref{section:standpoint-situated-conditionals} we introduce standpoint situated conditionals and show how we can increase the expressivity of PDSL, as well as rephrase the syntax of PDSL in terms of situated conditionals. In Section \ref{section:conditional-knowledge-bases} we define a fragment of PDSL which can be translated into propositional KLM knowledge bases, while in Section \ref{section:NMR-for-PDSL} we show how this translation can lift the class of rational defeasible entailment relations from the propositional case into the PDSL case.

\section{Preliminaries}\label{section:prelims}



In this section, we introduce Propositional Defeasible Standpoint Logics, as well as some basic definitions from  defeasible reasoning in the tradition of Kraus, Lehmann and Magidor (KLM) \cite{kraus:nonmonotonic}. In the KLM tradition of defeasible reasoning, the language $\LangProp$ consists of conditionals or defeasible implications of the form $\alpha\ptwiddle\beta$ where $\alpha$ and $\beta$ are Boolean formulas. We refer to (finite) subsets $\KB\subseteq \LangProp$ as propositional conditional knowledge bases. The semantics for $\LangProp$ are given by \textit{preferential} interpretations.

\begin{definition} \label{definition:propositional-preferential-interpretation}
   \cite{kraus:nonmonotonic} A \textit{preferential interpretation} over a set of propositional atoms $\mathcal{P}$ is a triple $\I=(W,l,\prec)$ where $W$ is a set of states, $l:W\longrightarrow 2^{\mathcal{P}}$ is a mapping from the set of states to the set of classical valuations on $\mathcal{P}$, and $\prec\subseteq W\times W$ is a strict partial order such that \textit{for every Boolean formula $\alpha$, the set $\llbracket \alpha \rrbracket_\I:=\{w\in W\mid l(w)\Vdash \alpha\}$  has a minimal element with respect to $\prec$. }
\end{definition}

We say that $\I\Vdash \alpha\twiddle\beta$, or $\I$ satisfies $\alpha\ptwiddle\beta$ if  $min_\prec( \llbracket\alpha\rrbracket)_\I)\subseteq \llbracket\beta\rrbracket_\I$. We extend this to propositional conditional knowledge bases in the usual way. Given a knowledge base $\KB\subseteq \LangProp$, we say that $\KB$ preferentially entails $\alpha\ptwiddle \beta$, written $\KB\vDash_{pref} \alpha\twiddle \beta$, if $\I\Vdash \KB$ implies $\I\Vdash \alpha\ptwiddle \beta$ for all preferential interpretations $\I$. However, as in the PDSL case in Example \ref{example:introductory-example}, $\vDash_{pref}$ is often considered too weak an inference operator for defeasible reasoning. This has led to several different notions of defeasible entailment to be defined for propositional conditional knowledge bases. An important class of these are rational defeasible entailment relations, defined by Lehmann and Magidor \cite{lehmann:conditionalentail}. These are the class of defeasible entailment relations defined by \textit{ranking functions}. Ranking functions were originally introduced by Spohn \cite{Spohn2012}. For our purposes, we consider a slightly restricted definition.

\begin{definition}\cite{casini:beyondratclosure}
 A \emph{ranking function} $\R$ is a function $\R:2^\mathcal{P}\longrightarrow \mathds{N}\cup \{\infty\}$, satisfying the following property: if $\R(u)<\infty$, then for every $0\leq j<r(u)$ there exists $v\in 2^\mathcal{P}$ such that $r(v)=j$.
\end{definition}

Each ranking function $\R$ induces a preferential interpretation $\I_\R=(W_\R,l_\R,\prec_\R)$ where $W_\R=2^\mathcal{P}\setminus\R^{-1}(\infty)$, $l_\R$ is the identity and $u\prec_\R v$ iff. $\R(u)<\R(v)$. We write $\R\Vdash \alpha\ptwiddle\beta$ to denote $\I_\R\Vdash \alpha\ptwiddle\beta$. 

\begin{definition}
    A rational defeasible entailment relation is an entailment relation $\dentails$ such that for any knowledge base $\KB$, there exists some ranking function $\R_\KB$ such that $\KB\dentails \alpha\ptwiddle\beta$ iff $\R_\KB\Vdash \alpha\ptwiddle\beta$.
\end{definition}

Many well-known entailment relations in defeasible reasoning are rational entailment relations. Later, we consider two of the most well-known in the literature: rational closure \cite{lehmann:conditionalentail,Pearl:SystemZ} and lexicographic closure \cite{lehmann:lexicographicreason}.
Casini et al. \cite{casini:beyondratclosure} study rational defeasible entailment in the general setting. They define a \textit{basic (rational) defeasible entailment} as a rational entailment relation $\dentails$ which satisfies \textit{inclusion} and \textit{classical preservation}. $\dentails$ satisfies inclusion if for all $\KB\subseteq\LangProp$ and all $\alpha\ptwiddle\beta\in \KB$, $\KB \dentails\alpha\ptwiddle\beta$. $\dentails$ satisfies classical preservation if: $\KB\dentails \alpha\ptwiddle \bot$ if and only if $\KB\vDash_{pref} \alpha\ptwiddle\bot$ for every Boolean formula $\alpha$. This is called classical preservation since $\I=(W,l,\prec)\Vdash \neg\alpha\ptwiddle\bot$ iff $w\Vdash \alpha$ for all $w\in W$.

\begin{definition}
    A ranking strategy is a mapping $r:2^{\LangProp}\longrightarrow\mathscr{R}$, where $\mathscr{R}$ is the set of ranked interpretations. That is, a ranking strategy assigns each propositional conditional knowledge base a ranking function. 
\end{definition}

Each ranking strategy $r:2^{\LangProp}\longrightarrow\mathscr{R}$ defines a rational defeasible entailment relation $\dentails_r$ by: $\KB\dentails_r \alpha\ptwiddle\beta$ iff $r(\KB)\Vdash \alpha\ptwiddle\beta$. We say a ranking strategy is well-defined if $\dentails_r$ is a basic defeasible entailment, and assume all ranking strategies considered in this paper are well-defined, unless otherwise stated. Casini et al. also define an algorithm $\mathtt{DefeasibleEntail}$ such that, given a knowledge base $\KB\subseteq\LangProp$, a query $\alpha\ptwiddle\beta\in\LangProp$, and a ranking strategy $r$, we have that $\KB\dentails_r\alpha\ptwiddle\beta$ iff $\mathtt{DefeasibleEntail}(\KB,r(\KB),\alpha\ptwiddle\beta)=\textit{True}$. In Section \ref{section:NMR-for-PDSL} we show that this algorithm can be used for entailment-checking in PDSL.  Inspired by the preferential semantics in propositional KLM, Propositional Defeasible Standpoint Logic (PDSL) was introduced by Leisegang et~al.~\cite{JELIA-NickIvanTommie}. It extends (classical) Propositional Standpoint Logic~\cite{alvarezrudolph:propositionalstdpt} with defeasible versions of both material implication \cite{kraus:nonmonotonic} and standpoint modalities in order to express not only the necessary or possible views of standpoints, but describes those views which are \textit{typical} or \textit{distinctly possible} for a given standpoint. These are drawn from the more general notion of defeasible modalities given by Britz and Varzinczak \cite{britzvarzin:defeasiblemodalities}. The syntax for PDSL is defined below.

\begin{definition} \cite{JELIA-NickIvanTommie}
        Given a vocabulary $\mathcal{V=(P,S)}$ where $\mathcal{P}$ is a set of propositional atoms and $\mathcal{S}$ is a set of standpoints, we define the set of standpoint expressions $\mathcal{E}$ over $\mathcal{S}$ as 
    \[e::= *\mid s\mid -e \mid e\cap e\]
where $s\in\mathcal{S}$. We define $\LangdSt$ over $\mathcal{V}$ (where $p\in\mathcal{P}$ and $e,d\in\mathcal{E}$) as follows:
    \[\alpha::=\top \mid p\mid e\lesssim d\mid\neg\alpha\mid  \alpha\wedge\alpha\mid \Box_e\alpha\mid{\dnec}_e\alpha\mid \alpha \leadsto \alpha\]
\end{definition}


We introduce the supplementary dual modalities as$\Diamond_e\alpha:=\neg\Box_e \neg \alpha$, $\dposs_e\alpha:=\neg\dnec_e\neg \alpha$ and classical ``standpoint sharpening'' statements as $e\leq d:=\Box_{e\cap - d}\bot$. Here, sentences of the form $\Box_e\alpha$ and $\Diamond_e \alpha$ are read as \textit{`` it is unequivocal to $e$ that $\alpha$''} and \textit{`` it is possible to $e$ that $\alpha$''}. On the other hand, sentences with defeasible modalities $\dnec_e\alpha$ and $\dposs_e \alpha$ are read as \textit{`` $\alpha$ holds in the most typical understandings of $e$'s viewpoint''} and \textit{`` it is distinctly possible to $e$ that $\alpha$''}. Standpoint sharpenings $e\leq d$ are read as \textit{``the standpoint $e$ is a more specific version of $d$''} and thus $e$ inherits the unequivocal beliefs of $d$. Defeasible standpoint sharpenings $e\lesssim d$ are read as \textit{``the most typical ways of understanding standpoint $e$ fall within standpoint $d$''}. The semantics for $\LangdSt$ is given by \textit{state-preferential standpoint structures} in which a preference ordering over the set of precisifications (possible worlds) is added to the semantics of classical standpoint logics \cite{alvarezrudolph:propositionalstdpt}. These structures and the satisfaction relation they introduce is defined below.

\begin{definition}\cite{JELIA-NickIvanTommie}\label{def:standpoint-structure}
    A state-preferential standpoint structure (SPSS) is a quadruple $M=(\Pi,\sigma,\tau,\prec)$ where,
    \begin{itemize}
        \item $\Pi$ is a set of precisifications (possible worlds).
        \item $\sigma:\mathcal{E}\longrightarrow 2^\Pi$ is a function such that $\sigma(*)=\Pi$,   $\sigma(e\cap d)=\sigma(e)\cap \sigma(d)$, and  $\sigma(-e)=\Pi\setminus\sigma(e)$. Moreover, we require that $\sigma(s)\neq \emptyset$ for all $s\in\mathcal{S}$.
        \item $\tau:\Pi\longrightarrow 2^\mathcal{P}$ is a map which assigns a classical valuation to each precisification.
        \item $\prec$ is a strict partial order on $\Pi$ such that for every subset $X$ of $\Pi$, and every $\pi\in X$, either $\pi$ is a $\prec$-minimal element of $X$, or there is a $\pi'\in X$ such that $\pi'$ is a $\prec$-minimal element of $X$, and $\pi'\prec\pi$ (well-foundedness).
    \end{itemize}
\end{definition}

Here, each precisification represents a possible precise set of views which may occur within the possible ways to interpret a standpoint. The map $\sigma$ allocates a set of precisifications to each standpoint, with the allocation being extended to standpoint expressions using set-theoretic operators. The preference ordering on $\Pi$ represents how ``typical'' each precisification is, where $\pi_1\prec \pi_2$ intuitively states that $\pi_1$ is a more typical precisification than $\pi_2$. The $\prec$-minimal precisifications in $\sigma(e)$ therefore represent the most typical ways to interpret the views of $e$, although there may be other, less usual ways to interpret this standpoint.

\begin{definition}\label{def:model-satisfaction-stdptstructure}
    For a state-preferential standpoint structure $M$ and a precisification $\pi\in\Pi$, we define the satisfaction relation $\Vdash$ as follows:
    \begin{itemize}
        \item $M,\pi\Vdash \top$.
        \item $M,\pi\Vdash p$ iff $p\in\tau(\pi)$.
        \item $M,\pi\Vdash e\lesssim d$ iff $\min_\prec(\sigma(e))\subseteq \sigma(d)$.
        \item $M,\pi\Vdash \neg \alpha$ iff $M,\pi\nVdash \alpha$.
         \item $M,\pi\Vdash \alpha_1\wedge \alpha_2$ iff $M,\pi\Vdash \alpha_1$ and $M,\pi\Vdash \alpha_2$.
         \item $M,\pi\Vdash \Box_s\alpha$ iff $M,\pi'\Vdash \alpha$ for all $\pi'\in \sigma(s)$.
         \item $M,\pi\Vdash \dnec_s\alpha$ iff $M,\pi'\Vdash \alpha$ for all $\pi'\in \min_\prec(\sigma(s))$.
        \item $M,\pi\Vdash \alpha_1\leadsto\alpha_2$ iff $\pi\notin\min_\prec\llbracket\alpha_1\rrbracket_M$ or $\pi\in\llbracket\alpha_2\rrbracket_M$.
         \item $M\Vdash\alpha$ iff $M,\pi\Vdash \alpha$ for all $\pi\in \Pi$.
    \end{itemize}
    where $\alpha,\alpha_1,\alpha_2\in\mathcal{L}_{\mathds{S}}$, $s,s_1,s_2\in\mathcal{S}$ and $p\in \mathcal{P}$ and $\llbracket\alpha\rrbracket_M:=\{\pi\in \Pi\mid M,\pi\Vdash \alpha\}$. Satisfaction is extended to sets of formulas in the usual way. 
\end{definition}

For both preferential interpretations and SPSSs, we may drop the $M$ in $\llbracket\alpha\rrbracket_M$ when it is clear which model we are referring to. Intuitively we have that $\dnec_e\alpha$ holds when the most typical precisifications assigned to $\sigma(e)$ satisfy $\alpha$, although other less-typical precisifications assigned to $e$ may not. A formula $\phi\in\LangdSt$ or knowledge base $\KB\subseteq\LangdSt$ are said to be (locally) satisfiable if there exists an SPSS $M$ and $\pi\in \Pi$ such that $M,\pi\Vdash \phi$ or $M,\pi\Vdash \KB$, respectively. A formula $\phi\in\LangdSt$ or knowledge base $\KB\subseteq\LangdSt$ are said to be (locally) satisfiable if there exists an SPSS $M$ such that $M\Vdash \phi$ or $M \Vdash \KB$, respectively. The following example depicts an SPSS.

\begin{example}\label{example:SPSS-that-shows-p-entailment-is-weak}
Consider the SPSS $M=(\Pi,\sigma,\tau,\prec)$ relating to the knowledge given in Example~\ref{example:introductory-example} shown in Figure~\ref{fig:SPSS-for-example-2}, where standpoints $\Ast$ and $\Phy$ are shortened to $\AstF$ and $\PhyF$, and atoms $\NL$ and $\VG$ are shortened to $\NLF$ and $\VGF$. Here $\Pi=\{\pi_1,\pi_2,\pi_3\}$; $\prec=\{(\pi_1,\pi_2),(\pi_2,\pi_3),(\pi_1,\pi_3)\}$; $\sigma(\AstF)=\{\pi_3\}$ and $\sigma(\PhyF)=\{\pi_2,\pi_3\}$; $\tau(\pi_1)=\{\NLF,\VGF\}$, $\tau(\pi_2)=\{\NLF\}$ and $\tau(\pi_3)=\emptyset$. Note here that $M\Vdash \AstF\lesssim\PhyF$, $M\Vdash {\dnec_\PhyF} \NLF$ and $M\Vdash \NLF\twiddle \VGF$. However, $M\nVdash {\dnec_\PhyF}\VGF$ and $M\nVdash {\dnec_\AstF} \NLF$.
\end{example}

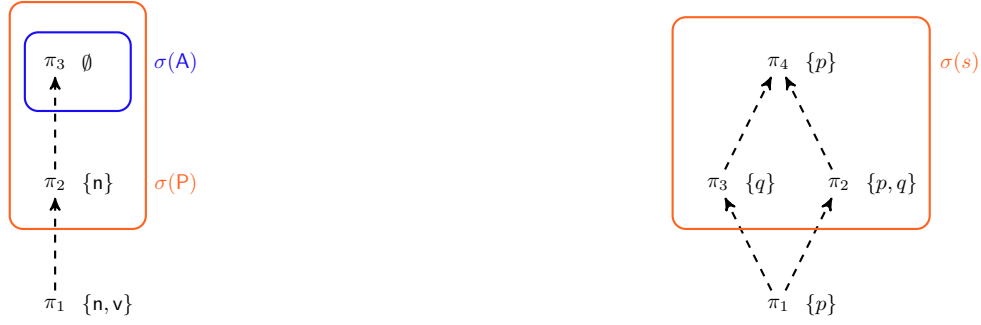
\begin{figure}[t]
\begin{subfigure}{0.4\linewidth}
\begin{center}
\scalebox{1}{
\begin{tikzpicture}
[->,>=stealth',thick, auto,node_style/.style={circle,fill=black,minimum size=0.1mm},scale=0.8, every node/.style={scale=0.8}]

\node (omega3) at (4,4) [label=right:\mbox{$\emptyset$}] {$\pi_{3}$} ;
\node (omega1) at (4,0)  [label=right:\mbox{$\{\NLF,\VGF\}$}] {$\pi_{1}$} ;
\node (omega2) at (4,2)  [label=right:\mbox{$\{\NLF\}$}] {$\pi_{2}$} ;



\path
(omega1) edge [dashed] (omega2)
(omega2) edge [dashed] (omega3);

\draw (3.25,5) [Orange,rounded corners=5pt,thick] -- (5.5,5) -- (5.5,1.25) -- (3.25,1.25) -- cycle ;

\draw (3.5,4.5) [BlueViolet,rounded corners=5pt,thick] -- (5.25,4.5) -- (5.25,3.2) -- (3.5,3.2) -- cycle ;

\node (sigmaAst) at (6,4) [ ] {\textcolor{BlueViolet}{$\sigma(\AstF)$}};

\node (sigmaAst) at (6,2) [ ] {\textcolor{Orange}{$\sigma(\PhyF)$}};

\end{tikzpicture}
}
\end{center}
\caption{SPSS in Example \ref{example:SPSS-that-shows-p-entailment-is-weak}}
\label{fig:SPSS-for-example-2}
\end{subfigure}
\hfill
\begin{subfigure}{0.4\linewidth}
\begin{center}
\scalebox{1}{
\begin{tikzpicture}
[->,>=stealth',thick, auto,node_style/.style={circle,fill=black,minimum size=0.1mm},scale=0.8, every node/.style={scale=0.8}]

\node (omega4) at (4,4) [label=right:\mbox{$\{p\}$}] {$\pi_{4}$} ;
\node (omega1) at (4,0)  [label=right:\mbox{$\{p\}$}] {$\pi_{1}$} ;
\node (omega3) at (3,2)  [label=right:\mbox{$\{q\}$}] {$\pi_{3}$} ;
\node (omega2) at (5,2)  [label=right:\mbox{$\{p,q\}$}] {$\pi_{2}$} ;



\path
(omega1) edge [dashed] (omega2)
(omega1) edge [dashed] (omega3)
(omega2) edge [dashed] (omega4)
(omega3) edge [dashed] (omega4);

\draw (2.25,4.75) [Orange,rounded corners=5pt,thick] -- (6.5,4.75) -- (6.5,1.25) -- (2.25,1.25) -- cycle ;



\node (sigmaAst) at (7,4) [ ] {\textcolor{Orange}{$\sigma(s)$}};

\end{tikzpicture}
}
\end{center}
\caption{SPSS in Example \ref{example:SPSS-for-situated-conditional-satisfiability.}}
\end{subfigure}%
\caption{Example SPSSes. Each precisification is labelled with its image under~$\tau$, and the order $\pi\prec\pi'$ is indicated by a dotted arrow from $\pi$ to $\pi'$, with transitive arrows removed.}
\label{fig:SPSS-for-example-3}
\end{figure}

In previous work \cite{JELIA-NickIvanTommie}, the computation of the monotonic Tarskian notion of entailment is characterised for PDSL. That is,
 given a PDSL knowledge base $\KB\subseteq \LangdSt$ and a formula $\phi\in\LangdSt$, we say that $\KB$ \textit{preferentially entails} $\phi$, or $\KB\vDash_P \phi$ if for any SPSS we have $M\Vdash \KB$ implies $M\Vdash \phi$. Example \ref{example:SPSS-that-shows-p-entailment-is-weak} therefore gives a us a countermodel to show that  $\{\AstF\lesssim\PhyF \dnec_\PhyF\NLF, \NLF\twiddle \VGF\}\nvDash_P\dnec_\PhyF\VGF$ and ${\{\AstF\lesssim\PhyF \dnec_\PhyF\NLF, \NLF\twiddle \VGF\}\nvDash_P\dnec_\AstF \NLF}$.

\section{Standpoint Situated Conditionals}\label{section:standpoint-situated-conditionals}


Within the framework of PDSL, we note that the defeasible implication $\twiddle$ allows us to contain propositional KLM-style conditionals within the standpoints context, but the $\leadsto$ symbol does not interact well with the elements of $\LangdSt$ which involve standpoint modalities or sharpening statements \cite{JELIA-NickIvanTommie}. The lack of nuance when adding a non-Boolean statement to a defeasible implication can be seen below.

\begin{proposition}\label{proposition:unintuitve-behaviour-with-modalities}

    Consider $\alpha,\beta,\gamma\in \LangdSt$ such that $\alpha= \#_e \gamma$ or $\alpha=e\lesssim d$ for $e,d\in \mathcal{E}$ and $\#_e\in \{\Box_e,\Diamond_e,\dnec_e,\dposs_e\}$. Then, for any SPSS $M$ we have that: 
    \begin{enumerate}
        \item If there is no $\pi\in\Pi$ such that $M,\pi\nVdash \alpha$, then $M\Vdash \alpha\leadsto \beta$. 
        \item If  there exists $\pi\in\Pi$ such that $M,\pi\Vdash \alpha$, then $M\Vdash \alpha \leadsto \beta \Leftrightarrow M\Vdash \top\leadsto \beta \Leftrightarrow M\Vdash {\dnec_*} \beta$.
        \item If ~$\llbracket\beta\rrbracket\neq\emptyset$, then $M\Vdash \beta\twiddle\alpha$ iff $M\Vdash \alpha$.
    \end{enumerate}

\end{proposition}

In order to integrate defeasible implications into standpoint logics in a more meaningful way, we propose an extension of this language inspired by \textit{situated conditionals}, introduced by Casini et al. \cite{CASINI-situated-conditionals}. A situated conditional takes the form $\alpha\ptwiddle_\gamma \beta$ for Boolean formulas $\alpha$ $\beta$ and $\gamma$: which reads as \textit{``given the situation $\gamma$, $\alpha$ usually implies that $\beta$''}. In our context, we introduce a similar idea, but instead of ``situating'' our conditional (defeasible implication) with an additional Boolean formula, we situate it within a certain standpoint. That is, we introduce conditionals of the form $\alpha\twiddle_e \beta$ where $\alpha$ and $\beta$ are Boolean, while $e$ is a standpoint symbol. This reads as \textit{``within the standpoint $e$, a belief in $\alpha$ usually implies a belief in $\beta$''}. In other words, if we are to consider situations where the standpoint $e$ holds $\alpha$ true, we can assume that $\beta$ is usually true, notwithstanding exceptional beliefs which may influence $e$ to hold $\alpha$ and $\neg \beta$ to be true simultaneously. Accordingly, we define the satisfiability of standpoint situational conditionals by:

\begin{equation}
    M,\pi\Vdash \alpha\twiddle_e\beta \textit{ iff }\pi\notin \min_\prec(\sigma(e)\cap \llbracket\alpha\rrbracket) \textit{ or } \pi\in \llbracket \beta\rrbracket .
\end{equation}

Then, if $M\Vdash \alpha\twiddle_e\beta$ we must have that  $\min_\prec(\sigma(e)\cap \llbracket\alpha\rrbracket) \subseteq\llbracket \beta\rrbracket$. That is, we require that the most preferred precisifications which both belong to standpoint $e$ and satisfy $\alpha$ must also satisfy $\beta$. 

\begin{example}\label{example:SPSS-for-situated-conditional-satisfiability.}
  Consider the SPSS depicted in Figure \ref{fig:SPSS-for-example-3} over $M=(\Pi,\sigma,\tau,\prec)$ over $\mathcal{S}=\{s\}$ and $\mathcal{P}=\{p,q\}$, where: $\Pi=\{\pi_1,\pi_2,\pi_3,\pi_4\}$; $\prec$ is denoted by the dotted arrows with transitive arrows omitted (e.g.~$\pi_2\prec\pi_4$); $\sigma(s)=\{\pi_2,\pi_3,\pi_4\}$ and $\tau(\pi_1)=\tau(\pi_4)=\{p\}$, $\tau(\pi_2)=\{p,q\}$, $\tau(\pi_3)=\{q\}$. Note here that while $M\nVdash p\twiddle q$, $M\nVdash \Box_s(p\rightarrow q)$ and $M\nVdash {\dnec_s} p$ we have that $M\Vdash p\twiddle_s q$ since $\min_\prec(\sigma(s)\cap\llbracket p\rrbracket) \subseteq\llbracket q\rrbracket$.
\end{example}

This intuitively blends the notion of defeasible implications with the additions of standpoints to propositional logic in a much more meaningful way. However, it also works to supersede notions already existing in PDSL.

\begin{proposition}\label{proposition:we-can-express-PDSL-in-terms-of-situated-conditionals}
    Let $M$ be an SPSS, $\alpha,\beta\in \LangdSt$ and $e\in \mathcal{E}$. Then,

    \begin{enumerate}
        \item $M,\pi\Vdash \alpha\twiddle\beta$ iff. $M,\pi\Vdash \alpha\twiddle_*\beta$.
        \item $M, \pi\Vdash \dnec_e \alpha$ iff. $M,\pi\Vdash \Box_*(\top \twiddle_e \alpha)$.
        \item $M, \pi\Vdash \Box_e \alpha$ iff. $M,\pi\Vdash \Box_*(\neg \alpha\twiddle_{e} \bot )$.
    \end{enumerate}
\end{proposition}

The last two utilise the fact that standpoint modalities are defined globally while defeasible implications (incl. standpoint situated conditionals) are defined locally. We can also express this result with the more intuitive form 
$M \Vdash \dnec_e \alpha$ iff. $M \Vdash \top \twiddle_e \alpha$ and $M \Vdash \Box_e \alpha$ iff. $M \Vdash \neg \alpha\twiddle_{e} \bot$. This means that in the general case, when including situated conditionals we can express the full language of $\LangdSt$ removing defeasible modalities and implications, and limiting classical standpoint modalities to only the universal standpoint $\Box_*$. We can then re-characterise the grammar of $\LangdSt$ by:

\begin{equation}\label{equation:new-PDSL-syntax}
 \alpha::=\top\mid p\mid e\lesssim d\mid \neg \alpha\mid \alpha\wedge \alpha\mid \Box_*\alpha\mid \alpha\twiddle_e\alpha\end{equation}




Another initial result is that $\twiddle_e$ is a faithful interpretation of a KLM-style condition. That is, $\twiddle_e$ satisfies the KLM postulates \cite{kraus:nonmonotonic}:

\begin{proposition}\label{proposition:situated-standpoint-condiitionals-satisfy-KLM-postulates}
    For any SPSS $M$, any Boolean formulas $\alpha,\beta$ and $\gamma$ and any standpoint symbol $e$, we have that:
    \begin{enumerate}
        \item $M\Vdash \alpha\twiddle_e\alpha$.
        \item If $M\Vdash \Box_e(\alpha\leftrightarrow \beta)$ and $M\Vdash \alpha\twiddle_e\gamma$ then $M\Vdash \beta\twiddle_e \gamma$.
        \item If $M\Vdash \alpha\twiddle_e \beta$ and $M\Vdash\Box_e(\beta\rightarrow\ \gamma)$ then $M\Vdash \alpha\twiddle_e \gamma$.
        \item If $M\Vdash \alpha\twiddle_e \beta$ and $M\Vdash \alpha\twiddle_e \gamma$ then $M\Vdash \alpha\twiddle_e \beta\wedge \gamma$    \item If $M\Vdash \alpha\twiddle_e \beta$ and $M\Vdash \gamma\twiddle_e \beta$ then $M\Vdash \alpha\vee \gamma\twiddle_e \beta$.
        \item If $M\Vdash \alpha\twiddle_e \beta$ and $M\Vdash \alpha\twiddle_e \gamma$ then  $M\Vdash \alpha\wedge \beta\twiddle_e \gamma$.
        \end{enumerate}
\end{proposition}

\section{Conditional Knowledge Bases}\label{section:conditional-knowledge-bases}


In this section, we describe a fragment of of PDSL which can be translated into a traditional KLM-style conditional knowledge base. This translation then can be utilised in order to characterise non-monotonic systems of inference or entailment for PDSL. The fragment of PDSL which we consider are those in which the formulas are some form of positive defeasible conditionals.

\begin{definition}
    The fragment $\textbf{Cond}(\LangdSt)\subseteq \LangdSt$  (for $e,d\in\mathcal{E}$ and $\alpha,\beta$ are Boolean) is the fragment of $\LangdSt$  consisting of all formulas of the form:
    \[\phi::=\alpha\twiddle_e \beta\mid e\lesssim d\]

    A PDSL knowledge base $\KB\subseteq \LangdSt$ is called a conditional knowledge base if $\KB\subseteq\textbf{Cond}(\LangdSt)$
\end{definition}

We should note here that although this logical fragment looks fairly restricted, it is sufficient enough to cover a fair portion of our language, especially since for entailment purposes we are interested in models which satisfy our knowledge base globally. In particular any Boolean formula $\alpha$ can be represented by $\neg \alpha\twiddle_* \bot$. Any defeasible implication, or Boolean formula bound by defeasible or classical necessity can be expressed, by Propositon \ref{proposition:we-can-express-PDSL-in-terms-of-situated-conditionals}. This includes classical standpoint sharpenings.
Furthermore, conjunctions of statements can be expressed by including each conjunct in the set. Therefore, our framework includes many of the core features of PDSL. The notable omissions stem from the absence of negation on the outer level. We cannot express distinctly in such a knowledge base that a defeasible implication, or a modal statement ought to \textit{not} hold. As a result, we also cannot describe non-Boolean disjunctive statements, and furthermore, we lose the ability to express both defeasible and classical diamond statements of the form $\Diamond_e\alpha$ or $\dposs_e\alpha$. As the name suggests, conditional knowledge bases in PDSL are constructed to be analogously structured to conditional knowledge bases in the propositional case \cite{lehmann:conditionalentail}. In fact, the central idea within this section is to show that we can translate conditional knowledge bases in PDSL to conditional knowledge bases in propositional logic.

\begin{definition}
    Consider the conditional language $\textbf{Cond}(\LangdSt)$ over the vocabulary $\mathcal{V}=(\mathcal{P},\mathcal{S})$. Then consider the language of propositional KLM $\LangProp$ defined over the atomic variables $\mathcal{P}'=\mathcal{P}\cup\mathcal{S}$. We define the translation $T:\textbf{Cond}(\LangdSt)\longrightarrow \LangProp$ as follows.
    \begin{itemize}
        \item We define the standpoint index translation $tr$ inductively over standpoint expressions by $t(s)=s$ for all $s\in \mathcal{S}$, $t(*)=\top$, $t(e\cap d)=t(e)\wedge t(d)$, and $t(-e)=\neg tr(e)$.
        \item $T(e\lesssim d)=tr(e)\ptwiddle tr(d)$.
        \item $T(\alpha\twiddle_e \beta)=\alpha\wedge tr(e)\ptwiddle \beta$.
    \end{itemize}

    We extend $T$ to knowledge bases in $\textbf{Cond}(\LangdSt)$ via the usual means. That is, $T(\KB):=\{T(\phi)\mid \phi\in \KB\}$ for $\KB\subseteq \textbf{Cond}(\LangdSt)$
\end{definition}


Note here that the translated sentence is no longer than the original sentence and computable in linearly many steps in the size of the original sentence. Abusing notation, we translate PDSL models into KLM propositional models with the same function $T$. Notice here how this translation essentially treats standpoint symbols as propositional atoms which may or may not be true at each precisification.

\begin{definition}
    Given an SPSS $M=(\Pi,\sigma,\tau,\prec)$, we define the propositional translated model of $M$ as $T(M)=(W',l',\prec')$ where
    \begin{itemize}
        \item $W'=\Pi$.
        \item $l(\pi)=\tau(\pi)\cup\{s\in \mathcal{S}\mid \pi\in\sigma(s)\}$.
        \item $\prec'=\prec$.
    \end{itemize}
\end{definition}

Our next result tells us that our translations between language and models correspond as expected.

\begin{proposition}\label{proposition:translation-of-PDSL-KB-into-Prop-KB}
    For any $\phi\in \textbf{Cond}(\LangdSt)$ and any SPSS $M$ we have that $M\Vdash \phi \textit{ iff. }T(M)\Vdash T(\phi)$.
\end{proposition}

This translation provides us with a context in which defining non-monotonic entailment relations for conditional knowledge bases in PDSL can be extrapolated from their propositional counterpart. The translation function $T$ is also injective on models, meaning that no two non-equivalent SPSSes induce identical preferential models.

\begin{proposition}\label{proposition:T-injective-on-models}
    $T$ is injective on models.
\end{proposition}

In fact, the image of $T$ in the set of preferential interpretations are exactly the set preferential interpretations $\I$ such that $\llbracket s\rrbracket_\I\neq \emptyset$ for all $s\in \mathcal{S}$. Clearly, it follows that $T$ is not surjective on models, since there are preferential interpretations over $\mathcal{P'}=\mathcal{P}\cup\mathcal{S}$ such that $\llbracket s\rrbracket_\I= \emptyset$ for some $s\in\mathcal{S}$. Moreover, for any preferential interpretation $\I=(W,l,\prec)$ in the image of $T$ we can construct its preimage $T^{-1}(\I):=(\Pi',\sigma',\tau',\prec')$ by: $\Pi'=W$; $\prec'=\prec$; for each $w\in W$, $\tau'(w)=l(w)\cap\mathcal{P}$; and for each $s\in \mathcal{S}$, $\sigma'(s)=\{w\in W\mid s\in l(w)\}$ and $\sigma'(*)=W$. $\sigma'$ is defined inductively on complex standpoint expressions. In our next section, we use this translation to lift methods of non-monotonic entailment from the propositional defeasible reasoning setting to the PDSL setting.




\section{Non-Monotonic Entailment}\label{section:NMR-for-PDSL}

In this section, we attempt to answer in the PDSL setting the same question posed for the propositional setting  by Lehmann and Magidor in \cite{lehmann:conditionalentail}: ``\textit{What does a conditional knowledge base entail?}'' 

In particular, rational defeasible entailments, which are defined semantically via a \textit{single ranked interpretation}. We first consider the general case of this question as investigated by et al. \cite{casini:beyondratclosure}, and then highlight two well-known forms of inference in the propositional case: rational closure \cite{lehmann:conditionalentail,Pearl:SystemZ} and lexicographic closure \cite{lehmann:lexicographicreason}. Given Proposition \ref{proposition:translation-of-PDSL-KB-into-Prop-KB}, a clear intuition is to build an SPSS model which characterises a non-monotonic entailment $\dentails$ via the translated model $T(M)$. That is, for a knowledge base $\KB\subseteq \textbf{Cond}(\LangdSt)$ we choose a model $M$ to characterise our entailment in the PDSL case, such that $T(M)$ induces the corresponding entailment result in the propositional case. Therefore, if we are given  $\KB\subseteq \textbf{Cond}(\LangdSt)$ and propositional defeasible entailment $\dentails_r$ with an associated ranking strategy $r:2^{\LangProp}\longrightarrow \mathscr{R}$ which we wish to define in the conditional-PDSL case, we proceed as follows: 

\begin{enumerate}
    \item Translate $\KB$ into its propositional form $T(\KB)$.
    \item Compute the semantic structure for $\dentails_r$ in the propositional case via $r(T(\KB))$.
    \item Using the fact that   nature of $T$ is injective on models, we define $T^{-1}(r(T(\KB))$ as the semantic model which characterises $\dentails$ for $\textbf{Cond}(\LangdSt)$. This relies on the following Lemma.
\end{enumerate}

\begin{lemma}\label{lemma:rankings-are-in-the-image} If $\KB$ is a satisfiable conditional PDSL knowledge base and $r$ is a well-defined ranking strategy, the preferential interpretation $r(T(\KB))$ is in the semantic image of ~$T$. 
\end{lemma}

We define $\dentails_r$ for conditional-PDSL as follows.

\begin{definition}
  Consider $\KB\subseteq \textbf{Cond}(\LangdSt)$, $\phi\in\LangdSt$ and a ranking strategy $r$. We say
${\KB\dentails_r \phi}$ \textit{ iff } $T^{-1}(r(T(\KB))\Vdash \phi$.
\end{definition}

It is worth noting here that the equivalence of models in the conditional-PDSL case and the original KLM case means that algorithms for this case can be used equivalently in the PDSL case. In the most general case for ranked interpretations, this is given by the $\texttt{DefeasibleEntail}$ algorithm~\cite{casini:beyondratclosure} which takes a ranking-based model in the propositional case and uses SAT-solving methods to compute whether a conditional is satisfied by this model. We therefore obtain the following corollary. 

\begin{corollary}
    Let $\dentails_r$ be the entailment associated with the ranking strategy $r$, and let $\KB$ be a $\textbf{Cond}(\LangdSt)$ knowledge base and $\phi\in \textbf{Cond}(\LangdSt)$. Then, $\KB\dentails_r \phi \text{ iff. } \mathtt{DefeasibleEntail}(T(\KB),r(T(\KB)), T(\phi))=\text{True}$.
\end{corollary}

Another upshot of this is that checking a query in  $\textbf{Cond}(\LangdSt)$  is in the same complexity class as the translated query for conditionals in propositional logic. However, although the entailment we have defined can only be computed with knowledge bases in  $\textbf{Cond}(\LangdSt)$, we are able to compute queries for larger fragments of $\LangdSt$. Specifically, we want to consider query-answering for a fragment of PDSL where we can answer queries with defeasible and classical diamond standpoint modalities. Using the syntax proposed in \ref{equation:new-PDSL-syntax} this can be formulated by:

\begin{definition}
     The fragment $\textbf{NegCond}(\LangdSt)\subseteq \LangdSt$ is the fragment of $\LangdSt$ consisting of all formulas of the form
    \[\phi::=\alpha\twiddle_e \beta\mid e\lesssim d\mid\ \naf\ \phi \]
where $\naf:=\neg \Box_*$, $\alpha$ and $\beta$ are Boolean formulas, and $e,d\in\mathcal{E}$ are standpoint expressions.
\end{definition}

 Note that the compound unary connective $\naf$ allows us to introduce classical or defeasible diamond statements into our language by Proposition \ref{proposition:we-can-express-PDSL-in-terms-of-situated-conditionals}, where $\Diamond_e\alpha$ is equivalent to $\naf (\alpha\twiddle_e \bot)$ and $\dposs_e\alpha$ is equivalent to $\naf(\top\twiddle_e \alpha)$. The syntax chosen for $\naf$ is used due to the similarity with negation-as-failure in the logic programming tradition. We note that $\naf$ is a strictly weaker notion than $\neg$ and that $\naf$ behaves according to the principle of negation as failure in our semantics.

\begin{proposition}\label{proposition:naf-proposition}
    Given an SPSS $M$ and $\phi\in \textbf{Cond}(\LangdSt)$ then we have:
    \begin{enumerate}
        \item $M\Vdash \neg \phi$ implies $M\Vdash \naf \phi$, and the converse does not hold in general.
        \item $M\Vdash \naf \phi$ iff. $M\nVdash \phi$.
    \end{enumerate}
\end{proposition}

Since for each ranking strategy $r$ we define the entailment $\dentails_r$ relative to a single SPSS, the symbol $\naf$ acts as negation as failure on the level of non-monotonic entailment as well.

\begin{corollary}
    Given a ranking strategy $r$, a formula $\phi\in \textbf{Cond}(\LangdSt)$ and a knowledge base $\KB\subseteq \textbf{Cond}(\LangdSt)$:  $\KB\dentails_r \naf \phi$ iff $\KB\ndentails_r\phi$ iff $\mathtt{DefeasibleEntail}(T(\KB),r_{T(\KB)}, T(\phi))=\text{False}$.
\end{corollary}


This in turn shows that $\naf \phi$ queries (including diamond modality queries) can be computed using a single call to the \texttt{DefeasibleEntail} algorithm, and that the query language can be expanded in this way with no effects to the complexity. We also note that we can compute queries which are conjunctions of formulas in $\textbf{NegCond}(\LangdSt)$ by just considering a conjunction as a set of its conjuncts and computing \texttt{DefeasibleEntail} for each such conjunct. Disjunctions of conditionals in general cannot be computed by evaluating each disjunct. For example, it may be the case, for an SPSS $M$, that  $M\Vdash \phi_1\vee\phi_2$ while both $M\nVdash \phi_1$ and $M\nVdash \phi_2$. This is not a feature specific to standpoint logics, but to the more general defeasible modal logic \textbf{K} as defined in \cite{britzvarzin:defeasiblemodalities}, and can be seen in the example below:

\begin{example}
    Consider the following model where the set of standpoint names is empty, $\mathcal{P}=\{p,q,r,s\}$, and $M=(\Pi,\tau,\prec)$ where $\Pi=\{\pi_1,\pi_2,\pi_3,\pi_4\}$; $\prec=\{(\pi_2,\pi_1),(\pi_4\pi_3)\}$; $\tau(\pi_1)=\{r\}$, $\tau(\pi_2)=\{q,r\}$, $\tau(\pi_3)=\{p\}$ and $\tau(\pi_4)=\{p,s\}$. Clearly, $M\nVdash p\twiddle_* q$ and $M\nVdash r\twiddle_* s$ since $\min_\prec \llbracket r\rrbracket\nsubseteq \llbracket s\rrbracket$ and  $\min_\prec \llbracket p\rrbracket\nsubseteq \llbracket q\rrbracket$. However, note that $M,\pi\Vdash ( p\twiddle_* q)\vee  (r\twiddle_* s)$ iff one of the following condition is satisfied: $\pi\notin\min_\prec \llbracket p\rrbracket$ or $\pi\in \llbracket q\rrbracket$ or $\pi\notin\min_\prec \llbracket r\rrbracket$ or $\pi\in \llbracket s\rrbracket$. Each element of $\Pi$ satisfies one of these conditions and so $M\Vdash ( p\twiddle_* q)\vee  (r\twiddle_* s)$.

\end{example}

The last thing we note in the general case is how queries of diamond-statements behave in a more ``permissive'' way than in classical propositional standpoint logic. In particular, in our definition of non-monotonic entailment for a ranking strategy $r$, we essentially include in a standpoint $e$ every valuation which is allowable from the knowledge base. That is, we conclude $\Diamond_e\alpha$ (for a Boolean statement $\alpha$) whenever $\Diamond_e \alpha$ is not directly precluded from the knowledge base. Therefore, although for classical Boolean statements and box statements our logic is supraclassical, when it comes to classical diamond queries, our logic answers queries more in the style of brave reasoning. This is formalised by the propositions below. We recall here that $\vDash_P$ is the Tarskian notion of preferential entailment defined at the end of Section \ref{section:prelims}.

\begin{proposition}\label{proposition:supraclassicality-for-box-and-Boolean-statements}
    Let $\KB$ be a conditional PDSL knowledge base, and let $\alpha$ be a Boolean statement, and $r$ be a well-defined ranking strategy. Then $\KB\dentails_r \alpha$ iff $\KB\vDash_P \alpha$. Furthermore, $\KB\dentails_r\Box_e \alpha$ iff $\KB\vDash_P\Box_e \alpha$.
\end{proposition}

\begin{proposition}\label{proposition:diamonds-behave-like-brave-reasoning}
    Let $\KB$ be a conditional PDSL knowledge base, and let $\alpha$ be a Boolean statement, and $r$ be a (valid) ranking strategy. Then $\KB\dentails_r \Diamond_e \alpha$ iff there exists a model of $\KB$, $M$, such that $M\Vdash \Diamond_e \alpha$.
\end{proposition}

We note that this is a departure from convention in conditional reasoning, where supraclassicality holds for all non-defeasible parts of the logic. However, we believe that brave reasoning for diamond bound statements is in the spirit of non-monotonic prototypical reasoning, since we assume standpoints to hold all beliefs as possible, unless we learn evidence which shows otherwise.
After considering the general case for ranking-based entailment relations in conditional PDSL, we now focus on two well known inference systems from the propositional case.


\subsection{Rational Closure}

\begin{figure}

\begin{center}
\begin{tikzpicture}
[->,>=stealth',thick, auto,node_style/.style={circle,fill=black,minimum size=0.1mm},scale=0.8, every node/.style={scale=0.8}]


\node (rank0) at (-2,-0.25) [] {\textbf{0}};
\node (omega1) at (0,0) [label=below:\mbox{$\{\NLF,\VGF\}$}] {$\bullet$} ;
\node (omega2) at (2,0) [label=below:\mbox{$\emptyset$}] {$\bullet$} ;
\node (omega3) at (4,0)  [label=below:\mbox{$\{\VGF\}$}] {$\bullet$};
\node (omega4) at (6,0)  [label=below:\mbox{$\{\NLF,\VGF\}$}] {$\bullet$} ;

\path  (-2.5,0.5) edge [dashed, -] (10.5,0.5);

\node (rank1) at (-2,1.25) [] {\textbf{1}};
\node (omega5) at (0,1.5) [label=below:\mbox{$\emptyset$}] {$\bullet$} ;
\node (omega6) at (2,1.5) [label=below:\mbox{$\{\NLF\}$}] {$\bullet$} ;
\node (omega7) at (4,1.5)  [label=below:\mbox{$\{\VGF\}$}] {$\bullet$};
\node (omega8) at (6,1.5)  [label=below:\mbox{$\emptyset$}] {$\bullet$} ;
\node (omega9) at (8,1.5)  [label=below:\mbox{$\{\NLF\}$}] {$\bullet$} ;
\node (omega10) at (10,1.5)  [label=below:\mbox{$\{\NLF\}$}] {$\bullet$} ;

\path  (-2.5,2) edge [dashed, -] (10.5,2);

\node (rank2) at (-2,2.75) [] {\textbf{2}};
\node (omega11) at (0,3) [label=below:\mbox{$\{\VGF\}$}] {$\bullet$} ;
\node (omega12) at (2,3) [label=below:\mbox{$\{\NLF,\VGF\}$}] {$\bullet$} ;
\node (omega13) at (4,3)  [label=below:\mbox{$\emptyset$}] {$\bullet$};
\node (omega14) at (6,3)  [label=below:\mbox{$\{\NLF\}$}] {$\bullet$} ;
\node (omega15) at (8,3)  [label=below:\mbox{$\{\VGF\}$}] {$\bullet$} ;
\node (omega16) at (10,3)  [label=below:\mbox{$\{\NLF,\VGF\}$}] {$\bullet$} ;

\draw (-1,3.5) [BlueViolet,rounded corners=5pt,thick] -- (11,3.5) -- (11,1.75) -- (9,1.75) -- (9,0.6) -- (5,0.6) -- (5,1.75)-- (-1,1.75)-- cycle ;

\draw (-0.8,3.25) [BurntOrange,rounded corners=5pt,thick] -- (3,3.25) -- (3,1.85) -- (8.5,1.85) -- (8.5,0.4) -- (1,0.4) -- (1,-1)-- (-0.8,-1)-- cycle ;

\node (sigmaEng) at (11.5,3.5) []{\textcolor{BlueViolet}{$\sigma(\EngF)$}};

\node (sigmaPhy) at (0,-1.5) []{\textcolor{BurntOrange}{$\sigma(\PhyF)$}};


\end{tikzpicture}

\end{center}

    \caption{SPSS for the Rational Closure of $\KB_1=\{\dnec_\PhyF \NLF, \NLF\twiddle\VGF, \EngF\lesssim\PhyF, \dnec_\EngF \neg \VGF\}$}
    \label{fig:PDSL-RC-Example}
\end{figure}

One of the most well-known and foundational rational entailment relations from propositional KLM style reasoning is rational closure \cite{lehmann:conditionalentail}, also known as System Z \cite{Pearl:SystemZ}. This system of reasoning rests semantically on the \textit{presumption of typicality}, which states that we should assume worlds in our models to be as typical as possible, unless our knowledge gives us reasons to believe otherwise. The semantic formalisation of rational closure can be characterised by a proposition given by Giordano et al. \cite{giordano:semanticRC}.

\begin{proposition}\cite{giordano:semanticRC}
    Let $\KB\subseteq \LangProp$ be a propositional conditional knowledge base. Then define the ordering $<$ on ranking functions over $\mathcal{P}$ by $\R_1<\R_2$ iff $\R_1(u)<\R_2(u)$ for all $u\in 2^\mathcal{P}$. The set of ranking functions which model $\KB$ $\mathscr{R}_\KB:=\{\R:2^\mathcal{P}\longrightarrow \mathbb{N}\cup\{\infty\}\mid \R\Vdash\KB \}$ has a unique $<$-minimal element. 
\end{proposition}

We denote this unique element as $\R^\KB_{RC}$ and use it to characterise the rational closure entailment relation.

\begin{definition}
    The entailment relation $\dentails_{RC}$ is the relation determined by the ranking strategy $r_{RC}:2^{\LangProp}\longrightarrow\mathscr{R}$ such that $r_{rc}(\KB)=\R^\KB_{RC}$ for all $\KB\subseteq \LangProp$.
\end{definition}

Using our definition from the general case, $\dentails_{RC}$ is adapted to the conditional-PDSL case as follows.

\begin{definition}
    For $\KB\subseteq \textbf{Cond}(\LangdSt)$ and $\phi\in \textbf{NegCond}(\LangdSt)$ we write that $\KB \dentails_{RC} \phi \textit{ iff }T^{-1}(\R^\KB_{RC})\Vdash \phi$
\end{definition}

\begin{example}\label{example:first-RC-eg}
    Consider the knowledge base related to our previous Example \ref{example:introductory-example}:

    \[\KB_1=\{\dnec_\PhyF \NLF, \NLF\twiddle\VGF, \EngF\lesssim\PhyF, \dnec_\EngF \neg \VGF\}\]
    where $\PhyF$ and $\EngF$ denotes the physics and engineering standpoints respectively; $\NLF$ and $\VGF$ denotes the beliefs in Newton's laws and variable gravitational force, respectively. Translating this knowledge base via  Proposition \ref{proposition:translation-of-PDSL-KB-into-Prop-KB}, we obtain $T(\KB_1)=\{\PhyF\ptwiddle \NLF, \NLF\ptwiddle\VGF, \EngF\ptwiddle\PhyF, \EngF\ptwiddle\neg \VGF\}$. Figure \ref{fig:PDSL-RC-Example} depicts the SPSS which determines $\dentails_{RC}$ for $\KB_1$ in the PDSL case, determined by calculating $r_{RC}(T(\KB_1))$ and finding its preimage under the semantic translation function $T$. In the figure, each precisification is not given a unique name but denoted with a bullet point. The image of each precisification under $\tau$ is labelled as in Example \ref{example:SPSS-that-shows-p-entailment-is-weak}. The sets of $\sigma(\EngF)$ and $\sigma(\PhyF)$ are given by the labelled outlines in the diagrams. Finally, the ordering $\prec$ is given by a representative listing of ``ranks'' for precisifications on the left hand side of the diagram emulating the ranking function which induces the SPSS in the propositional case. For example, for any $\pi$ in line $0$ and any precisification $\pi'$ in line $1$ we have $\pi\prec\pi'$. As we can see in Figure \ref{fig:PDSL-RC-Example}, there is a unique minimal precisification $\pi$ in $\sigma(P)$ where $\tau(\pi)=\{\NLF,\VGF\}$. Thus, unlike in the case of $\vDash_P$, we can conclude $\KB_1\dentails_{RC} \dnec_\PhyF \VGF$. This shows that rational closure is indeed inferentially stronger than $\vDash_P$ for PDSL, and in particular, it avoids the problems shown in Example \ref{example:SPSS-that-shows-p-entailment-is-weak}, since in the absence of conflicting information, we can now conclude that $\PhyF$ inherits the belief in $\VGF$ from the sentences $\dnec_\PhyF \NLF$ and $\NLF\twiddle\VGF$. Although not included for space purposes, it is also the case that $\{\dnec_\PhyF \NLF,\Ast\lesssim\PhyF\}\dentails_{RC} \dnec_{\Ast}\NLF$, showing that sub-standpoints can inherit defeasibly held beliefs under rational closure, when there is no conflicting information.     On the other hand, we can see in Figure \ref{fig:PDSL-RC-Example} that the engineering perspective is usually a sub-standpoint of the physics standpoint, but as specified in $\KB_1$, does not inherit all the usual beliefs of ~ $\PhyF$. Simply from the knowledge in $\KB_1$, it is also unclear whether $\EngF$'s typical belief in $\neg\VGF$ is in conflict with the premise $\NLF$. On one hand, $\NLF\twiddle \VGF$ may intuitively suggest that a rejection of $\VGF$ is linked to a rejection of the overall premise of $\NLF$. On the other hand, it may be the case that the standpoint $\EngF$ rejects $\VGF$ as a conclusion of $
    \NLF$, but does not reject $\NLF$ itself (since $\NLF\twiddle \VGF$ is a defeasible implication). From the knowledge in $\KB_1$, this is hard to distinguish, and so from the model we conclude $\KB_1\ndentails_{RC} \dnec_\EngF \NLF$ since we cannot conclude that the engineering perspective still believes in Newton's laws as a whole. However, we are able to conclude that $\KB_1\dentails_{RC} \dposs_\EngF \NLF$. That is, we can say that although there is not evidence that the engineering perspective typically accepts $\NLF$, there is still a \emph{distinct possibility }that $\NLF$ holds.
    \end{example}

    The above example provides an intuitive look at the semantics and the extended expressive power of rational closure in the PDSL case. However, it is also worth noting the computational properties of rational closure in PDSL. As previously mentioned, we are able to utilise algorithms from the propositional case, and can use $\mathtt{DefeasibleEntail}$ in the most general case. In the case of rational closure, we are able to instead check rational closure in the PDSL case with a single call to the well-known $\mathtt{RationalClosure}$ algorithm \cite{freund:preferentilreasoning,lehmann:conditionalentail} which is $\textsc{P}^{\textsc{NP}}_\parallel$-complete \cite{lehmann:conditionalentail, LucasiewiczEiter:ComplexityFromDefaultReasoning}. Moreover, the translation from PDSL into propositional conditionals needed to make this call is linear. As a corollary of this we get the following.

    \begin{corollary}
  For $\KB\subseteq \textbf{Cond}(\LangdSt)$ and $\phi\in \textbf{NegCond}(\LangdSt)$, checking whether $\KB\dentails_{RC} \phi$ is $\textsc{P}^{\textsc{NP}}_\parallel$-complete.
    \end{corollary}

\subsection{Lexicographic Closure}

\begin{figure}

\begin{center}
\begin{tikzpicture}
[->,>=stealth',thick, auto,node_style/.style={circle,fill=black,minimum size=0.1mm},scale=0.8, every node/.style={scale=0.8}]


\node (rank0) at (-2,-0.25) [] {\textbf{0}};
\node (omega1) at (0,0) [label=below:\mbox{$\{\inertia,\VGF\}$}] {$\bullet$} ;
\node (omega2) at (2,0) [label=below:\mbox{$\emptyset$}] {$\bullet$} ;
\node (omega3) at (4,0)  [label=below:\mbox{$\{\VGF\}$}] {$\bullet$};
\node (omega4) at (6,0)  [label=below:\mbox{$\{\inertia,\VGF\}$}] {$\bullet$} ;
\node (omega10) at (8,0)  [label=below:\mbox{$\{\inertia\}$}] {$\bullet$} ;

\path  (-2.5,0.5) edge [dashed, -] (10.5,0.5);

\node (rank1) at (-2,1.25) [] {\textbf{1}};
\node (omega5) at (0,1.5) [label=below:\mbox{$\emptyset$}] {$\bullet$} ;
\node (omega6) at (2,1.5) [label=below:\mbox{$\{\inertia\}$}] {$\bullet$} ;
\node (omega7) at (4,1.5)  [label=below:\mbox{$\{\VGF\}$}] {$\bullet$};
\node (omega8) at (6,1.5)  [label=below:\mbox{$\emptyset$}] {$\bullet$} ;
\node (omega9) at (8,1.5)  [label=below:\mbox{$\{\inertia\}$}] {$\bullet$} ;

\path  (-2.5,2) edge [dashed, -] (10.5,2);

\node (rank2) at (-2,2.75) [] {\textbf{2}};
\node (omega11) at (0,3) [label=below:\mbox{$\{\VGF\}$}] {$\bullet$} ;
\node (omega12) at (2,3) [label=below:\mbox{$\{\inertia,\VGF\}$}] {$\bullet$} ;
\node (omega13) at (4,3)  [label=below:\mbox{$\emptyset$}] {$\bullet$};
\node (omega14) at (6,3)  [label=below:\mbox{$\{\inertia\}$}] {$\bullet$} ;
\node (omega15) at (8,3)  [label=below:\mbox{$\{\VGF\}$}] {$\bullet$} ;
\node (omega16) at (10,3)  [label=below:\mbox{$\{\inertia,\VGF\}$}] {$\bullet$} ;

\draw (-1,3.5) [BlueViolet,rounded corners=5pt,thick] -- (11,3.5) -- (11,1.75) -- (9,1.75) -- (9,0.6) -- (5,0.6) -- (5,1.75)-- (-1,1.75)-- cycle ;

\draw (-0.8,3.25) [BurntOrange,rounded corners=5pt,thick] -- (3,3.25) -- (3,1.85) -- (8.5,1.85) -- (8.5,0.4) -- (1,0.4) -- (1,-1)-- (-0.8,-1)-- cycle ;

\node (sigmaEng) at (11.5,3.5) []{\textcolor{BlueViolet}{$\sigma(\EngF)$}};

\node (sigmaPhy) at (0,-1.5) []{\textcolor{BurntOrange}{$\sigma(\PhyF)$}};


\end{tikzpicture}

\end{center}

    \caption{SPSS for the Rational Closure of $\KB_2=\{\dnec_\PhyF \inertia, \dnec_\PhyF \VGF, \EngF\lesssim\PhyF, \dnec_\EngF \neg \VGF\}$}
    \label{fig:PDSL-RC-Bad-Example}
\end{figure}

In the propositional case for KLM-style defeasible reasoning, rational closure can again be considered too restrictive. In the PDSL case, a similar case can be made.

\begin{example}
    Consider the following knowledge base similar to that in Example \ref{example:first-RC-eg}
    \[\KB_2=\{\dnec_\PhyF \inertia, \dnec_\PhyF \VGF, \EngF\lesssim\PhyF, \dnec_\EngF \neg \VGF\}\]

   Here instead of specifying standpoint $\PhyF$ believing in Newton's laws and variable gravitational force ($\VGF$) being a defeasible conclusion of this, we specify directly that the physics standpoint usually holds $\VGF$ ($\dnec_\PhyF \VGF$). We also specify that the physics standpoint holds another of Newton's laws to be true, the law of inertia $(\inertia)$: which states that objects will not change motion or direction unless acted upon by a force. Hence, we have $\dnec_\PhyF \inertia$. We once again have in $\KB_2$ that the engineering standpoint is usually a sharpening of the physics standpoint, but from an engineering perspective, gravitational force is usually not considered a variable. $\KB_2$ and $\KB_1$ are fairly similar. However, while a belief in Newton's laws $(\NLF)$ and variable gravitational force $(\VGF)$ is intuitively and logically connected, a beliefs $(\VGF)$ and $(\inertia)$ are not intuitively connected, since they are two separate laws of physics concerning different properties. They are also not logically connected in $\KB_2$. Therefore, while we may expect not to conclude $\dnec_\EngF \NLF$ from $\KB_1$, it seems reasonable to conclude $\dnec_\EngF \inertia$ from $\KB_2$. Although the engineering standpoint is exceptional within the physics standpoint with respect to $\VGF$, there is nothing to suggest it should be exceptional with respect to $\inertia$. Since engineering is still usually a sub-standpoint of physics, it should stand to inherit the beliefs of the physics standpoint where there is no contradictory information.   However, when computing the rational closure of $\KB_2$, we obtain Figure \ref{fig:PDSL-RC-Bad-Example}, which is almost identical to Figure \ref{fig:PDSL-RC-Example}, up to replacing $\NLF$ with $\inertia$. In particular, we have that $\KB_2\ndentails_{RC}\dnec_\EngF \inertia$, and so rational closure is not nuanced enough to consider the logical differences between $\KB_1$ and $\KB_2$.
\end{example}

The above example shows an instance in PDSL of the \textit{drowning problem} \cite{BenferhatCDLP93}, a well-known issue for rational closure in the propositional case.  Here, the fact that the engineering standpoint is atypical in the physics standpoint with respect to $\VGF$ blocks the engineering from inheriting any other typical beliefs from physics, even though they may not be relevant to the reason engineers are atypical in the standpoint of physics. It has been shown in the propositional case that various well-known systems of rational entailment avoid the drowning problem (see for example \cite{Jesse-Drowning}). Here, we show that one of the best known entailment relations with this property, lexicographic closure \cite{lehmann:conditionalentail} is able to circumvent this problem in the PDSL case as well. Lexicographic closure has a number of characterizations. We use the one given in \cite{casini:beyondratclosure}. For a propositional conditional knowledge base  $\KB\subseteq \LangProp$, let $C^\KB:2^\mathcal{P}\longrightarrow\mathbb{N}$ be a function defined by $C^\KB(u)=|\{\alpha\ptwiddle\beta\mid u\Vdash \alpha\rightarrow\beta \}|$, where $|\cdot|$ denotes the cardinality of the set.

\begin{definition}
   Let $\KB\subseteq\LangProp$ be a propositional conditional knowledge base over $\mathcal{P}$. Let $C^\KB:2^\mathcal{P}\longrightarrow\mathbb{N}$ be a function defined by $C^\KB(u)=|\{\alpha\ptwiddle\beta\mid u\Vdash \alpha\rightarrow\beta \}|$. We define the preferential interpretation $\I_{LC}^{\KB}=(W,l,\prec^\KB_{LC})$ by $W=2^\mathcal{P}\setminus \{u\in 2^\mathcal{P}\mid {R}_{RC}^\KB(u)=\infty\}$, $l$ is the identity and $v\prec^\KB_{LC} u$ if $\mathcal{R}_{RC}^\KB(v)<\mathcal{R}_{RC}^\KB(u)$, or $\mathcal{R}_{RC}^\KB(v)=\mathcal{R}_{RC}^\KB(u)$ and $C^\KB(v)>C^\KB(u)$.
\end{definition}

   We say $\KB\dentails_{LC}\phi$ for $\phi\in\LangProp$ iff $\I_{LC}^{\KB}\Vdash \phi$. Furthermore, we note that $\I_{LC}^{\KB}$ can be expressed by a ranking function $\R_{LC}^{\KB}:2^\mathcal{P}\longrightarrow\mathbb{N}\cup\{\infty\}$ where $\R_{LC}^{\KB}(u)=0$ iff $u$ is $\prec_{LC}^{\KB}$-minimal in $W$, $\R_{LC}^{\KB}(u)=\infty$ if  $\R_{RC}^{\KB}(u)=\infty$, and $\R_{LC}^{\KB}(u)$ is the length of the $\prec_{LC}^{\KB}$-path from any $\prec_{LC}^{\KB}$-minimal element $v$ to $u$, otherwise. We use the above definition to specify lexicographic closure for conditional PDSL.

\begin{definition}
    For $\KB\subseteq \textbf{Cond}(\LangdSt)$ and $\phi\in \textbf{NegCond}(\LangdSt)$ we write that $\KB \dentails_{LC} \phi \textit{ iff }T^{-1}(\I^\KB_{LC})\Vdash \phi$.
\end{definition}

\begin{figure}

\begin{center}
\begin{tikzpicture}
[->,>=stealth',thick, auto,node_style/.style={circle,fill=black,minimum size=0.1mm},scale=0.8, every node/.style={scale=0.8}]


\node (rank0) at (-2,-0.25) [] {\textbf{0}};
\node (omega1) at (0,0) [label=below:\mbox{$\{\inertia,\VGF\}$}] {$\bullet$} ;
\node (omega2) at (2,0) [label=below:\mbox{$\emptyset$}] {$\bullet$} ;
\node (omega3) at (4,0)  [label=below:\mbox{$\{\inertia\}$}] {$\bullet$};
\node (omega4) at (6,0) [label=below:\mbox{$\{\VGF\}$}] {$\bullet$} ;
\node (omega5) at (8,0)  [label=below:\mbox{$\{\inertia,\VGF\}$}] {$\bullet$} ;

\path  (-2.5,0.5) edge [dashed, -] (9.5,0.5);

\node (rank1) at (-2,1.25) [] {\textbf{1}};
\node (omega6) at (0,1.5) [label=below:\mbox{$\{\inertia\}$}] {$\bullet$} ;
\node (omega7) at (2,1.5)  [label=below:\mbox{$\{\VGF\}$}] {$\bullet$};
\node (omega8) at (4,1.5)  [label=below:\mbox{$\{\inertia\}$}] {$\bullet$} ;

\path  (-2.5,2) edge [dashed, -] (9.5,2);

\node (rank2) at (-2,2.75) [] {\textbf{2}};
\node (omega9) at (0,3)  [label=below:\mbox{$\emptyset$}] {$\bullet$} ;
\node (omega10) at (2,3)  [label=below:\mbox{$\emptyset$}] {$\bullet$} ;

\path  (-2.5,3.5) edge [dashed, -] (9.5,3.5);

\node (rank3) at (-2,4.25) [] {\textbf{3}};
\node (omega11) at (0,4.5) [label=below:\mbox{$\{\inertia,\VGF\}$}] {$\bullet$} ;
\node (omega12) at (2,4.5) [label=below:\mbox{$\{\inertia\}$}] {$\bullet$} ;
\node (omega13) at (4,4.5)  [label=below:\mbox{$\emptyset$}] {$\bullet$};

\path  (-2.5,5) edge [dashed, -] (9.5,5);

\node (rank4) at (-2,5.75) [] {\textbf{4}};
\node (omega14) at (0,6)  [label=below:\mbox{$\{\VGF\}$}] {$\bullet$} ;
\node (omega15) at (2,6)  [label=below:\mbox{$\{\inertia,\VGF\}$}] {$\bullet$} ;
\node (omega16) at (4,6)  [label=below:\mbox{$\{\VGF\}$}] {$\bullet$} ;

\draw (-1,6.5) [BlueViolet,rounded corners=5pt,thick] -- (5,6.5) -- (5,0.6) -- (3,0.6) -- (3,2.1) -- (1,2.1) -- (1,3.6)-- (-1,3.6)-- cycle ;

\draw (-0.8,6.25) [BurntOrange,rounded corners=5pt,thick] -- (1.2,6.25) -- (1.2,3.25) -- (2.8,3.25) -- (2.8,2.2) -- (4.8,2.2) -- (4.8,0.4)-- (1,0.4) --(1,-1)-- (-0.8,-1)-- cycle ;

\node (sigmaEng) at (6,6) []{\textcolor{BlueViolet}{$\sigma(\EngF)$}};

\node (sigmaPhy) at (0,-1.5) []{\textcolor{BurntOrange}{$\sigma(\PhyF)$}};


\end{tikzpicture}

\end{center}

    \caption{SPSS for the Lexicographic Closure of $\KB_2=\{\dnec_\PhyF \inertia, \dnec_\PhyF \VGF, \EngF\lesssim\PhyF, \dnec_\EngF \neg \VGF\}$}
    \label{fig:PDSL-LC-Example}
\end{figure}

\begin{example}
    The SPSS characterizing the lexicographic closure of $\KB_2$ is given in Figure \ref{fig:PDSL-LC-Example}. As we can see, the ordering is more fine-grained than that in Figure \ref{fig:PDSL-RC-Bad-Example}. In particular, in Figure \ref{fig:PDSL-LC-Example} there is only one $\prec$-minimal precisification $\pi$ in $\sigma(\EngF)$ with $\tau(\pi)\Vdash \inertia$. Therefore, we have $\KB_2\dentails_{LC}~\dnec_\EngF\inertia$~ as desired.
\end{example}

Hence, lexicographic closure helps us avoid the drowning problem in PDSL. We obtain as corollaries some other well-known results from propositional setting. Firstly,  lexicographic closure is stronger than rational closure in the following sense: for $\KB\subseteq\LangProp$ and $\phi\in\LangProp$ we have that $\KB\dentails_{RC}\phi$ implies that $\KB\dentails_{LC}\phi$ \cite{lehmann:lexicographicreason}. This result relies on the fact that the order on the lexicographic model is a refinement of that in the rational closure model. Since the same model-theoretic relationship holds via the translation to the conditional PDSL case we obtain the following. For a conditional PDSL knowledge base $\KB\subseteq\textbf{Cond}(\LangdSt)$ and $\phi\in\textbf{NegCond}(\LangdSt)$: $\KB\dentails_{RC}\phi$ implies that $\KB\dentails_{LC}\phi$. Once again, we can use a single call to the propositional algorithm for computing lexicographic closure in order to determine lexicographic closure in the PDSL case. The algorithm for lexicographic closure is known to be $\textsc{P}^{\textsc{NP}}$-complete \cite{LucasiewiczEiter:ComplexityFromDefaultReasoning}. Therefore, this complexity is preserved for our case.

    \begin{corollary}
  For $\KB\subseteq \textbf{Cond}(\LangdSt)$ and $\phi\in \textbf{NegCond}(\LangdSt)$, checking whether $\KB\dentails_{LC} \phi$ is $\textsc{P}^{\textsc{NP}}$-complete.
    \end{corollary}

\section{Related Work and Conclusions}

Another approach to integrating KLM-style defeasibility and standpoint logics is given by Defeasible Restricted Standpoint logic (DRSL), introduced by Leisegang et al. \cite{LeisegangRudolphMeyer:SacairStandpoints}. DRSL captures standpoints which hold beliefs that are defeasible conditionals in nature, while PDSL captures classical beliefs which are held defeasibly by standpoints. In particular, the preference ordering in PDSL semantics is defined over the set of precisifications $\Pi$, whereas in DRSL semantics, there is no ordering on $\Pi$ and instead each precisification $\pi\in\Pi$ is mapped to a preferential interpretation. As such, DRSL semantics do not express defeasible modalities, while PDSL was constructed explicitly in order to express such modalities. Gorczyca and Straß \cite{GS2026} consider another method of non-monotonic reasoning with standpoint modalities, by integrating standpoints and the non-monotonic modal logic S4F. Preferential semantics and defeasible modalities have been considered in the modal logic \textbf{K} by Britz et al.  \cite{britzvarzin:defeasiblemodalities} and in linear temporal logic by Chafik et al. \cite{chafik:defeasiblelineartemporal}, although these focus on monotonic entailment. Outside of modal logics, KLM-style rational defeasible entailments have been also been lifted to Description Logics. Britz et al. \cite{britz2020principles} consider the general case, while rational closure \cite{giordano:semanticRC,britz2020principles,PenselTurhan18-defeasibleEL} and lexicographic closure \cite{casini2012lexicographic,HaldmiannMeyerCasini-LexWDLs-KR2025} have also been considered specifically.

In this paper, we have proposed a systematic means to lift rational defeasible entailment relations from the propositional case to a fragment of PDSL. In particular, we made the case that the preferential entailment introduced in PDSL by Leisegang et al. \cite{JELIA-NickIvanTommie} was unnecessarily inferentially weak, and sought to remedy this by introducing non-monotonic inferences from the propositional case into PDSL. In doing so, we enhanced the expressivity of PDSL through \textit{standpoint situated conditionals}, showing that the syntax of PDSL may be rephrased in terms of these conditionals. We then identified the fragment $\textbf{Cond}(\LangdSt)$ of PDSL consisting of just conditionals and showed that this fragment can be translated into traditional KLM conditional knowledge bases, while preserving model satisfaction. Using this translation, we lifted the class of rational defeasible entailment relations to  $\textbf{Cond}(\LangdSt)$ in the general case, before focussing on the well-known systems of rational and lexicographic closure in the conditional-PDSL case. In all these cases, we showed that propositional algorithms can be used to compute the conditional-PDSL case, and hence that the complexity of reasoning in conditional-PDSL is the same as in the propositional case.

\begin{acknowledgments}
This work is based on the research supported in part by the National Research Foundation of South Africa (REFERENCE NO: SAI240823262612). The construction of standpoint situated conditionals was inspired by discussions with Hannes Straß. Thanks also to the anonymous reviewers for their comments.
\end{acknowledgments}

\section*{Declaration on Generative AI}
  The authors have not employed any Generative AI tools.

\bibliography{defeasiblestandpointmodalities}

\newpage
\appendix

\section{Appendix: Omitted Proofs}

\noindent\textbf{Proposition \ref{proposition:unintuitve-behaviour-with-modalities}}     \textit{     Consider $\alpha,\beta,\gamma\in \LangdSt$ such that $\alpha= \#_e \gamma$ or $\alpha=e\lesssim d$ for $e,d\in \mathcal{E}$ and $\#_e\in \{\Box_e,\Diamond_e,\dnec_e,\dposs_e\}$. Then, for any SPSS $M$ we have that: 
    \begin{enumerate}
        \item If there is not $\pi\in\Pi$ such that $M,\pi\nVdash \alpha$, then $M\Vdash \alpha\leadsto \beta$. 
        \item If  there exists $\pi\in\Pi$ such that $M,\pi\Vdash \alpha$, then $M\Vdash \alpha \leadsto \beta \Leftrightarrow M\Vdash \top\leadsto \beta \Leftrightarrow M\Vdash {\dnec_*} \beta$.
        \item If ~$\llbracket\beta\rrbracket\neq\emptyset$, then $M\Vdash \beta\twiddle\alpha$ iff $M\Vdash \alpha$.
    \end{enumerate}}

\begin{proof}

We first note that $M\Vdash \alpha\twiddle\beta$ iff. for any $\pi\in\Pi$ $\pi\notin \min_\prec\llbracket\alpha\rrbracket$ or$ \pi\in \llbracket\beta\rrbracket$. That is, iff $\min_\prec \llbracket\alpha\rrbracket\subseteq \llbracket\beta\rrbracket$.

    \begin{enumerate}
        \item If there is no $\pi\in \Pi$ with $M,\pi\Vdash \alpha$ then $\llbracket\alpha\rrbracket=\emptyset$ and so $\min_
    \prec\llbracket\alpha\rrbracket=\emptyset$. It follows trivially that $\min_\prec \llbracket\alpha\rrbracket\subseteq \llbracket\beta\rrbracket$.
    \item If there exists $\pi\in \Pi$ such that $M,\pi\Vdash \alpha$ we note the following: satisfaction for statements of the form $\alpha$ are defined in terms of precisifications in $\sigma(e)$ and do not depend on $\pi$. For example, if $\alpha=\dnec_e\gamma$ then $M,\pi\Vdash \dnec_e\gamma$ iff $M,\pi'\Vdash \gamma$ for all $\pi'\in\min_\prec \sigma(e)$. Since this condition does not depend on the choice original precisification $\pi$, we have that $M,\pi\Vdash \alpha$ iff. $M\Vdash \alpha$. The proof that this holds for the other cases of $\alpha$ are similar. Therefore, if $M,\pi\Vdash \alpha$ then $M\Vdash \alpha$ and so $\llbracket\alpha\rrbracket=\Pi$. By definition, we also have $\Pi=\sigma(*)=\llbracket\top\rrbracket$ Hence, $M\Vdash \alpha\twiddle\beta$ iff. $\min_\prec\sigma(*)=\min_\prec \llbracket\top\rrbracket=\min_\prec \llbracket\alpha\rrbracket\subseteq \llbracket\beta\rrbracket$. The condition that $\min_\prec \llbracket\top\rrbracket\subseteq \llbracket\beta\rrbracket$ is equivalent to $M\Vdash \top\twiddle\beta$, and the condition $\min_\prec\sigma(*)\subseteq \llbracket\beta\rrbracket$ is equivalent to $M\Vdash \dnec_*\beta$.
    \item Clearly, if $M\Vdash \alpha$ then $M\Vdash\beta\twiddle\alpha$. On the other hand $M\Vdash \beta\twiddle \alpha$ iff $\min_\prec\llbracket\beta\rrbracket\subseteq \llbracket\alpha\rrbracket$. Since $\llbracket\beta\rrbracket$ is non-empty, then for the above to hold we must have that $\llbracket\alpha\rrbracket$ is non-empty. However, from proof of part 2. of the theorem, we know that if $\llbracket\alpha\rrbracket$ is non-empty, then $\llbracket\alpha\rrbracket=\Pi$ and so $M\Vdash \alpha$.
    \end{enumerate}
\end{proof}

\noindent\textbf{Proposition \ref{proposition:we-can-express-PDSL-in-terms-of-situated-conditionals} }\textit{ Let $M$ be an SPSS, $\alpha,\beta\in \LangdSt$ and $e\in \mathcal{E}$. Then,
\begin{enumerate}
        \item $M,\pi\Vdash \alpha\twiddle\beta$ iff. $M,\pi\Vdash \alpha\twiddle_*\beta$.
        \item $M, \pi\Vdash \dnec_e \alpha$ iff. $M,\pi\Vdash \Box_*(\top \twiddle_e \alpha)$.
        \item $M, \pi\Vdash \Box_e \alpha$ iff. $M,\pi\Vdash \Box_*(\neg \alpha\twiddle_{e} \bot )$.
    \end{enumerate} }

    \begin{proof}
        \begin{enumerate}
            \item Note that since $\sigma(*)=\Pi$ we have that $\llbracket\alpha\rrbracket\cap\sigma(*)=\llbracket\alpha\rrbracket$. Hence, $M,\pi\Vdash \alpha\twiddle\beta$ iff $\pi\notin \min_\prec(\llbracket\alpha\rrbracket)=\min_\prec(\llbracket\alpha\rrbracket\cap\sigma(*)$ or $\pi\in \llbracket\beta\rrbracket$. Equivalently, $M,\pi\Vdash \alpha\twiddle_*\beta$.
            \item $M,\pi\Vdash \Box_*(\top\twiddle_e\alpha)$ iff $M,\pi'\Vdash \top\twiddle_e\alpha$ for all $\pi'\in \Pi$. This is the case iff $\min_\prec(\llbracket\top\rrbracket\cap\sigma(e))\subseteq \llbracket\alpha\rrbracket$. Since $\llbracket\top\rrbracket=\Pi$, this is equivalent to $\min_\prec(\sigma(e))\subseteq \llbracket\alpha\rrbracket$. Then finally note that $\min_\prec(\sigma(e))\subseteq \llbracket\alpha\rrbracket$ is equivalent to $M,\pi''\Vdash \alpha$ for all $\pi''\in \min_\prec\sigma(e)$. That is, $M,\pi\Vdash \dnec_e\alpha$.
            \item Note that, similar to the previous case, $M,\pi\Vdash \Box_*(\neg \alpha\twiddle_{e} \bot )$ iff $\min_\prec(\sigma(e)\cap\llbracket\neg\alpha\rrbracket)\subseteq \llbracket\bot\rrbracket$. Then, since $\llbracket\bot\rrbracket=\emptyset$ it follows that $\sigma(e)\cap\llbracket\neg\alpha\rrbracket=\emptyset$. Hence, for any $\pi'\in \sigma(e)$ we must have that $\pi'\notin\llbracket\neg\alpha\rrbracket$. Therefore $\tau(\pi')\Vdash \alpha$. Hence, $M,\pi'\Vdash \alpha$ for all $\pi'\in \sigma(e)$ and so $M,\pi\Vdash \Box_e\alpha$. On the other hand, if $M,\pi\Vdash \Box_e\alpha$ then we must have that $\sigma(e)\subseteq \llbracket\alpha\rrbracket$. Hence $\sigma(e)\cap\llbracket\neg\alpha\rrbracket=\emptyset=\llbracket\bot\rrbracket$, and so $M\Vdash \neg\alpha\twiddle_e \bot$. Equivalently, $M\Vdash \Box_*(\neg\alpha\twiddle_e \bot)$ and so $M,\pi\Vdash \Box_*(\neg\alpha\twiddle_e \bot)$.
        \end{enumerate}
    \end{proof}

    \noindent\textbf{Proposition \ref{proposition:situated-standpoint-condiitionals-satisfy-KLM-postulates}}\textit{     For any SPSS $M$, any Boolean formulas $\alpha,\beta$ and $\gamma$ and any standpoint symbol $e$ we have that:
    \begin{enumerate}
        \item $M\Vdash \alpha\twiddle_e\alpha$.
        \item If $M\Vdash \Box_e(\alpha\leftrightarrow \beta)$ and $M\Vdash \alpha\twiddle_e\gamma$ then $M\Vdash \beta\twiddle_e \gamma$.
        \item If $M\Vdash \alpha\twiddle_e \beta$ and $M\Vdash\Box_e(\beta\rightarrow\ \gamma)$ then $M\Vdash \alpha\twiddle_e \gamma$.
        \item If $M\Vdash \alpha\twiddle_e \beta$ and $M\Vdash \alpha\twiddle_e \gamma$ then $M\Vdash \alpha\twiddle_e \beta\wedge \gamma$    \item If $M\Vdash \alpha\twiddle_e \beta$ and $M\Vdash \gamma\twiddle_e \beta$ then $M\Vdash \alpha\vee \gamma\twiddle_e \beta$.
        \item If $M\Vdash \alpha\twiddle_e \beta$ and $M\Vdash \alpha\twiddle_e \gamma$ then  $M\Vdash \alpha\wedge \beta\twiddle_e \gamma$.
        \end{enumerate}}

    \begin{proof}
    Let $M=(\Pi,\sigma,\tau,\prec)$ be an SPSS:
        \begin{enumerate}
            \item In any SPSS $\min_\prec(\sigma(e)\cap\llbracket\alpha\rrbracket)\subseteq \llbracket\alpha\rrbracket$ and so $M\Vdash \alpha\twiddle_e\alpha$.
            \item If $M\Vdash \alpha\twiddle_e \gamma$ we have $\min_\prec(\sigma(e)\cap\llbracket\alpha\rrbracket)\subseteq \llbracket\gamma\rrbracket$. Then, since $M\Vdash \Box_e(\alpha\leftrightarrow\beta)$ we have that $\tau(\pi)\Vdash \alpha\leftrightarrow\beta$ for all $\pi\in\sigma(e)$. Hence, if $\pi\in \sigma(e)$ then $\tau(\pi)\Vdash \alpha$ iff $\tau(\pi)\Vdash \beta$. That is, $\sigma(e)\cap\llbracket\alpha\rrbracket=\sigma(e)\cap\llbracket\beta\rrbracket$. It follows that $\min_\prec(\sigma(e)\cap\llbracket\beta\rrbracket)=\min_\prec(\sigma(e)\cap\llbracket\alpha\rrbracket)\subseteq \llbracket\gamma\rrbracket$ and so $M\Vdash \beta \twiddle_e\gamma$.
            \item If $M\Vdash \Box_e(\beta\rightarrow\gamma)$ we have that $\tau(\pi)\Vdash \beta\rightarrow \gamma$ for all $\pi\in \sigma(e)$. Since $M\Vdash \alpha \twiddle_e \beta$ we have $\min_\prec(\sigma(e)\cap\llbracket\alpha\rrbracket)\subseteq \llbracket\beta\rrbracket$. But then, for any $\pi'\in\min_\prec(\sigma(e)\cap\llbracket\alpha\rrbracket)$ we have that $\tau(\pi')\Vdash \beta$. But then, since $\pi'\in\sigma(e)$ we have that $\tau(\pi')\Vdash \beta\rightarrow\gamma$ and so $\tau(\pi')\Vdash \gamma$. That is, $\min_\prec(\sigma(e)\cap\llbracket\alpha\rrbracket)\subseteq \llbracket\gamma\rrbracket$ and so $M\Vdash \alpha\twiddle_e\gamma$.
            \item If $M\Vdash \alpha\twiddle_e \beta$ and $M\Vdash \alpha\twiddle_e \gamma$ we have that $\min_\prec(\sigma(e)\cap\llbracket\alpha\rrbracket)\subseteq \llbracket\beta\rrbracket$ and $\min_\prec(\sigma(e)\cap\llbracket\alpha\rrbracket)\subseteq \llbracket\gamma\rrbracket$. It is well known that $\llbracket\beta\rrbracket\cap\llbracket\gamma\rrbracket=\llbracket\beta\wedge\gamma\rrbracket$ and so $\min_\prec(\sigma(e)\cap\llbracket\alpha\rrbracket)\subseteq \llbracket\beta \wedge\gamma\rrbracket$. Therefore $M\Vdash \alpha\twiddle_e \beta\wedge \gamma$.
            \item Suppose that $\min_\prec(\sigma(e)\cap\llbracket\alpha\rrbracket)\subseteq \llbracket\beta\rrbracket$ and $\min_\prec(\sigma(e)\cap\llbracket\gamma\rrbracket)\subseteq \llbracket\beta\rrbracket$. We have from well-known results that $\sigma(e)\cap\llbracket\alpha \vee \gamma\rrbracket=\sigma(e)\cap(\llbracket\alpha\rrbracket\cup \llbracket\gamma\rrbracket)=((\sigma(e)\cap\llbracket\alpha\rrbracket)\cup (\sigma(e)\cap\llbracket\gamma\rrbracket))$. Then if $\pi\in \min_\prec(\sigma(e)\cap\llbracket\alpha \vee \gamma\rrbracket)$ it follows that $\pi\in \min_\prec((\sigma(e)\cap\llbracket\alpha\rrbracket)\cup (\sigma(e)\cap\llbracket\gamma\rrbracket))\subseteq (\sigma(e)\cap\llbracket\alpha\rrbracket)\cup (\sigma(e)\cap\llbracket\gamma\rrbracket)$. 
            
            Now suppose that $\pi\notin \min_\prec(\sigma(e)\cap\llbracket\alpha\rrbracket)$. Then $\pi\notin\sigma(e)\cap\llbracket\alpha\rrbracket$, since if it were, it would not be minimal and so there would exist $\pi'$ such that $\pi'\in\sigma(e)\cap\llbracket\alpha\rrbracket\subseteq (\sigma(e)\cap\llbracket\alpha\rrbracket)\cup (\sigma(e)\cap\llbracket\gamma\rrbracket)$ and $\pi'\prec \pi$, which contradicts that  $\pi\in \min_\prec((\sigma(e)\cap\llbracket\alpha\rrbracket)\cup (\sigma(e)\cap\llbracket\gamma\rrbracket))$. Hence, $\pi\in \sigma(e)\cap\llbracket\gamma\rrbracket$. Moreover, $\pi\in \min_\prec(\sigma(e)\cap\llbracket\gamma\rrbracket)$ since if it were not minimal in $\sigma(e)\cap\llbracket\gamma\rrbracket$ it would once again contradict that $\pi$ is $\prec$-minimal in $(\sigma(e)\cap\llbracket\alpha\rrbracket)\cup (\sigma(e)\cap\llbracket\gamma\rrbracket)$. 
            
            Therefore, if $\pi\in \min_\prec((\sigma(e)\cap\llbracket\alpha\rrbracket)\cup (\sigma(e)\cap\llbracket\gamma\rrbracket))$ either $\pi\in \min_\prec(\sigma(e)\cap\llbracket\alpha\rrbracket)\subseteq \llbracket\beta\rrbracket$ or $\pi\in \min_\prec(\sigma(e)\cap\llbracket\gamma\rrbracket)\subseteq \llbracket\beta\rrbracket$. In either case, $\pi\in\llbracket\beta\rrbracket$. That is, $\min_\prec(\sigma(e)\cap\llbracket\alpha \vee \gamma\rrbracket)= \min_\prec((\sigma(e)\cap\llbracket\alpha\rrbracket)\cup (\sigma(e)\cap\llbracket\gamma\rrbracket))\subseteq \llbracket\beta\rrbracket$. Hence $M\Vdash \alpha\vee\gamma\twiddle_e\beta$.

            \item If $M\Vdash \alpha\twiddle_e\beta$ and $M\Vdash \alpha\twiddle_e\gamma$ then we have that $\min_\prec(\sigma(e)\cap\llbracket\alpha\rrbracket)\subseteq \llbracket\beta\rrbracket$ and $\min_\prec(\sigma(e)\cap\llbracket\alpha\rrbracket)\subseteq \llbracket\gamma\rrbracket$. Consider $\pi\in \min_\prec(\sigma(e)\cap\llbracket\alpha\wedge \beta\rrbracket)=\min_\prec(\sigma(e)\cap\llbracket\alpha\rrbracket\cap\llbracket \beta\rrbracket)\subset \sigma(e)\cap\llbracket\alpha\rrbracket$. If $\pi\notin \min_\prec(\sigma(e)\cap\llbracket\alpha\rrbracket)$ then since $\pi\in \sigma(e)\cap\llbracket\alpha\rrbracket$ we must have that there exists $\pi'\in \min_\prec(\sigma(e)\cap\llbracket\alpha\rrbracket)$ such that $\pi'\prec \pi$. However, since $\min_\prec(\sigma(e)\cap\llbracket\alpha\rrbracket)\subseteq \llbracket\beta\rrbracket$  it follows that $\pi'\in \sigma(e)\cap\llbracket\alpha\rrbracket\cap\llbracket \beta\rrbracket$. However, this contradicts the assumption that $\pi\in \min_\prec(\sigma(e)\cap\llbracket\alpha\rrbracket\cap\llbracket \beta\rrbracket)$. Therefore, we must have that $\pi\in\min_\prec(\sigma(e)\cap\llbracket\alpha\rrbracket)$. Therefore,  $\min_\prec(\sigma(e)\cap\llbracket\alpha\rrbracket\cap\llbracket \beta\rrbracket)\subseteq \min_\prec(\sigma(e)\cap\llbracket\alpha\rrbracket)\subseteq \llbracket\gamma\rrbracket$ and so we have that $M\Vdash \alpha\wedge \beta\twiddle_e\gamma$.
        \end{enumerate}
    \end{proof}

    \noindent\textbf{Proposition \ref{proposition:translation-of-PDSL-KB-into-Prop-KB}} \textit{ For any $\phi\in \textbf{Cond}(\LangdSt)$ and any SPSS $M$ we have that 
    \[M\Vdash \phi \textit{ iff. }T(M)\Vdash T(\phi)\]}

    \begin{proof}
        We begin with a preliminary lemma that shows that for each $e\in\mathcal{E}$, $\sigma(e)=\llbracket t(e)\rrbracket_{T(M)}$. This is proved inductively:
        \begin{itemize}
            \item If $e=s$ for $s\in \mathcal{S}$ then $\llbracket t(s)\rrbracket_{T(M)}=\llbracket s\rrbracket_{T(M)}$. Then $\pi\in \llbracket s\rrbracket_{T(M)}$ iff $s\in l'(\pi)$ which is true iff $\pi\in \sigma(s)$. Hence, $\llbracket t(s)\rrbracket_{T(M)}=\sigma(s)$.
            \item If $e=*$ then $\llbracket t(*)\rrbracket_{T(M)}=\llbracket \top\rrbracket_{T(M)}=\Pi=\sigma(*)$.
            \item Let $e=g\cap d$ for $g,d\in\mathcal{E}$. Then  $\llbracket t(g\cap d)\rrbracket_{T(M)}= \llbracket t(g)\wedge t(d)\rrbracket_{T(M)}=\llbracket t(g)\rrbracket_{T(M)}\cap \llbracket(d)\rrbracket_{T(M)}$. Then by our inductive hypothesis $\llbracket t(g)\rrbracket_{T(M)}\cap \llbracket(d)\rrbracket_{T(M)}=\sigma(g)\cap\sigma(d)=\sigma(g\cap d)$.
            \item If $e=-g$ then $\llbracket t(-g)\rrbracket_{T(M)}= \llbracket \neg t(g)\rrbracket_{T(M)}=\Pi\setminus \llbracket t(g)\rrbracket_{T(M)}$. By inductive hypothesis $\Pi\setminus \llbracket t(g)\rrbracket_{T(M)}=\Pi\setminus \sigma(g)=\sigma(-g)$.
        \end{itemize}
    Therefore for each $e\in\mathcal{E}$, $\sigma(e)=\llbracket t(e)\rrbracket_{T(M)}$.
    
        For our main result we consider two cases
        \begin{itemize}
            \item[a. ] If $\phi=\alpha\twiddle_e\beta$ then $T(\phi)=\alpha\wedge t(e)\ptwiddle \beta$. Thus $M\Vdash \phi$ iff we have that $\min_\prec(\sigma(e)\cap\llbracket\alpha\rrbracket_M \subseteq \llbracket\beta\rrbracket_M$. Clearly, for any Boolean formula $\alpha$ with atoms in $\mathcal{P}$, $\llbracket\alpha\rrbracket_M=\llbracket\alpha\rrbracket_{T(M)}$. Therefore $\sigma(e)\cap\llbracket\alpha\rrbracket_M =\llbracket  t(e)\rrbracket_{T(M)}\cap\llbracket\alpha\rrbracket_{T(M)}=\llbracket t(e)\wedge\alpha\rrbracket_{T(M)}$, and $\llbracket\beta\rrbracket_M=\llbracket\beta\rrbracket_{T(M)}$. Furthermore, since $\prec'=\prec$ we have that $\min_\prec(\sigma(e)\cap\llbracket\alpha\rrbracket_M \subseteq \llbracket\beta\rrbracket_M$ is equivalent to $\min_{\prec'}(\llbracket t(e)\wedge\alpha\rrbracket_{T(M)})\subseteq \llbracket\beta\rrbracket_{T(M)}$. This in turn is equivalent to $T(M)\Vdash t(e)\wedge \alpha\ptwiddle\beta$.
            \item[b.] If $\phi=e\lesssim d$ then $T(\phi)=t(e)\ptwiddle t(d)$. We gave that $M\Vdash \phi$ iff $\min_\prec\sigma(e)\subseteq \sigma(d)$. By our preliminary lemma and the fact that $\prec'=\prec$, this is equivalent to $\min_\prec'\llbracket t(e)\rrbracket_{T(M)}\subseteq \llbracket t(d) \rrbracket_{T(M)}$. Finally, this is equivalent to $T(M)\Vdash t(e)\ptwiddle t(d)$.
        \end{itemize}
    \end{proof}

    \noindent\textbf{Proposition \ref{proposition:T-injective-on-models}. } \textit{$T$ is injective on models.}

    \begin{proof}
        Suppose that $(W,l,\prec'')=T(M)=T(M')$ where $M=(\Pi,\sigma,\tau,\prec)$ and $M'=(\Pi',\sigma',\tau',\prec')$. Then we have by definition that $\Pi'=W=\Pi$ and $\prec'=\prec''=\prec$. Then, from the definition of $T(M)$, it is clear that for any $s\in \mathcal{S}$, $\pi\in \sigma(s)$ iff $s\in l(\pi)$. Similarly, $\pi\in \sigma'(s)$ iff $s\in l(\pi)$. Therefore $\sigma(s)=\sigma'(s)$ for all $s\in\mathcal{S}$. Finally, for any $\pi\in\Pi$, $\tau(\pi)=l(\pi)\setminus\{s\in\mathcal{S}\mid \pi\in\sigma(s)\}=l(\pi)\setminus\{s\in\mathcal{S}\mid \pi\in\sigma'(s)\}=\tau'(\pi)$ and we are done.
    \end{proof}

    \noindent\textbf{Lemma \ref{lemma:rankings-are-in-the-image}. }\textit{ If $\KB$ is a satisfiable conditional PDSL knowledge base and $r$ is a well-defined ranking strategy, the preferential interpretation $r(T(\KB))$ is in the semantic image of $T$. }

\begin{proof}
    Since $r$ is a well-defined defeasible ranking strategy, then $r(T(\KB))$ satisfies classical preservation. In particular, for any $s\in\in\mathcal{S}$ we have that  $r(T(\KB))\Vdash s\ptwiddle \bot$ iff. $T(\KB)\vDash_{pref} s\ptwiddle \bot$. That is, every preferential interpretation of $T(\KB)$ satisfies $s\twiddle\bot$.
    
    However, since $\KB$ is satisfiable, there exists an SPSS $M=(\Pi,\sigma,\tau,\prec)$ such that $M\Vdash \KB$. Since by definition $\sigma(s)\neq\empty$. That is, $M\nVdash \Box_s\bot$ or equivalently (by Proposition \ref{proposition:we-can-express-PDSL-in-terms-of-situated-conditionals}) $M\nVdash \top\twiddle_s\bot$. Now note that by Proposition \ref{proposition:translation-of-PDSL-KB-into-Prop-KB} we have that $T(M)\Vdash T(\phi)$ for all $\phi\in \KB$. That is, $T(M)$ is a model of $T(\KB)$. Proposition \ref{proposition:translation-of-PDSL-KB-into-Prop-KB} also tells us that $M\nVdash \top\twiddle_s\bot$ implies that $T(M)\nVdash T(\top\twiddle_s\bot)=\top\wedge s\ptwiddle \bot$. Equivalently $T(M)\nVdash s\ptwiddle \bot$. Therefore, since $T(M)$ is a model of $T(\KB)$ we must have that $T(\KB)\nvDash_{pref} s\ptwiddle \bot$ and by classical preservation $r(T(\KB))\nVdash s\ptwiddle \bot$ for any $s\in\mathcal{S}$. Thus, $\llbracket s\rrbracket_{r(T(\KB))}\neq \emptyset$ for any $s\in\mathcal{S}$ and so $r(T(\KB))$ is in the image of $T$.
\end{proof}

\noindent\textbf{Proposition \ref{proposition:naf-proposition}. }\textit{ Given an SPSS $M$ and $\phi\in \textbf{Cond}(\LangdSt)$ then we have:
    \begin{enumerate}
        \item $M\Vdash \neg \phi$ implies $M\Vdash \naf \phi$, and the converse does not hold in general.
        \item $M\Vdash \naf \phi$ iff. $M\nVdash \phi$.
    \end{enumerate}}

    \begin{proof}

    \begin{enumerate}
        \item If $M\Vdash \neg \phi$ then $M,\pi\Vdash \neg\phi$ for all $\pi\in \Pi$. Since $\Pi$ is non-empty (Definition \ref{def:standpoint-structure}), we can pick $\pi$ such that $M,\pi\Vdash \neg\phi$. Then, $M\Vdash \neg \Box_*\phi$ is equivalent to $M\Vdash \Diamond_*\neg \phi$ which is equivalent to saying that there exists $\pi'\in \sigma(*)=\Pi$ such that $M,\pi'\Vdash \neg \phi$. Then, since $M,\pi\Vdash \neg\phi$ this condition is satisfied. To show that the converse does not hold consider an SPSS $M$ such that $\Pi=\{\pi_1,\pi_2\}$ and $\tau(\pi_1)=\emptyset$ and $\tau(\pi_2)=\{p\}$. Then, $M\nVdash \neg(\bot\twiddle_* \neg p)$ since $M\Vdash \neg(\bot\twiddle_* \neg p)$is equivalent to $M\Vdash \neg p$ and clearly $M,\pi_2\nVdash \neg p$. However, $M\Vdash \neg\Box_*(p)$ since there exists $\pi_1\in \Pi$ such that $M,\pi_1\Vdash \neg p$.
        \item $M\nVdash \phi$ iff there exists some $\pi\in\Pi$ such that $M,\pi\nVdash \phi$, which is equivalent to $M,\pi\Vdash \neg \phi$. Then, as previously mentioned in the first part of the proof, $M\Vdash \naf \phi$ iff there is some $\pi\in \Pi$ such that $M,\pi\Vdash \neg \phi$ and so we are done.    
    \end{enumerate}
\end{proof}

\noindent\textbf{Proposition \ref{proposition:supraclassicality-for-box-and-Boolean-statements}. }\textit{ Let $\KB$ be a conditional PDSL knowledge base, and let $\alpha$ be a Boolean statement, and $r$ be a well-defined ranking strategy. Then $\KB\dentails_r \alpha$ iff $\KB\vDash_P \alpha$. Furthermore, $\KB\dentails_r\Box_e \alpha$ iff $\KB\vDash_P\Box_e \alpha$.}

\begin{proof}
    Clearly if $\KB\vDash_P \phi$ then $\KB\dentails_r \phi$ for any PDSL statement $\phi$. 

    For the converse, we consider the following: by Proposition \ref{proposition:we-can-express-PDSL-in-terms-of-situated-conditionals} we have that $T^{-1}(r(T(\KB))\Vdash \alpha$ iff $T^{-1}(r(T(\KB))\Vdash\neg\alpha\twiddle_* \bot$ and $T^{-1}(r(T(\KB))\Vdash \Box_e\alpha$ iff $T^{-1}(r(T(\KB))\Vdash\neg \alpha\twiddle_e \bot$. Using our translation and Proposition \ref{proposition:translation-of-PDSL-KB-into-Prop-KB} this is equivalent to $r(T(\KB))\Vdash \neg \alpha\ptwiddle\bot$ and  $r(T(\KB))\Vdash \neg \alpha\wedge t(e) \ptwiddle\bot$ respectively. By the classical preservation of $r$, we can once more reduce the above conditions to $T(\KB)\vDash_{pref} \neg \alpha\ptwiddle\bot$ and $T(\KB)\vDash_{pref} \neg \alpha\wedge t(e)\ptwiddle\bot$ respectively.

    Considering the above, for any SPSS $M$ which satisfies $\KB$ (i.e. $M\Vdash \KB$) we have that $T(M)\Vdash T(\KB)$ by Proposition \ref{proposition:translation-of-PDSL-KB-into-Prop-KB}. Hence $T(M)$ is a model of $T(\KB)$ and so by the previous paragraph $T(M)\Vdash \neg \alpha\ptwiddle\bot$ and $T(M)\Vdash \neg \alpha\wedge t(e) \ptwiddle\bot$. Once again by Proposition \ref{proposition:translation-of-PDSL-KB-into-Prop-KB} we obtain that this is equivalent to $M\Vdash \neg \alpha\twiddle \bot$ and $M\Vdash \neg \alpha\twiddle_e \bot$ respectively. Lastly, since this holds for any model of $\KB$, $M$, this in turn gives us $\KB\vDash_P \neg \alpha\twiddle \bot$ and $\KB\vDash_P \neg \alpha\twiddle_e \bot$ respectively, as desired.
\end{proof}

\noindent\textbf{Proposition \ref{proposition:diamonds-behave-like-brave-reasoning}. }\textit{ Let $\KB$ be a conditional PDSL knowledge base, and let $\alpha$ be a Boolean statement, and $r$ be a (valid) ranking strategy. Then $\KB\dentails_r \Diamond_e \alpha$ iff there exists a model of $\KB$, $M$, such that $M\Vdash \Diamond_e \alpha$.}

\begin{proof}
    Clearly, if $\KB\dentails_r \Diamond_e \alpha$ we have that $T^{-1}(r(T(\KB))$ is a model of $\KB$ such that $T^{-1}(r(T(\KB))\Vdash \Diamond_e \alpha$.

    We then prove the converse. Assume that there is some SPSS $M$ such that $M\Vdash \KB$ and $M\Vdash \Diamond_e \alpha$. Then, it follows that $M\nVdash \Box_e \neg \alpha$, or equivalently $M\nVdash \alpha\twiddle_e \bot$. By Proposition \ref{proposition:translation-of-PDSL-KB-into-Prop-KB}  we have that $T(M)\nVdash \alpha \wedge t(e)\ptwiddle \bot$. Since $T(M)$ is a model of $T(\KB)$ (by Proposition \ref{proposition:translation-of-PDSL-KB-into-Prop-KB}) then $T(\KB)\nvDash_{pref}\alpha \wedge t(e)\ptwiddle \bot$. Then by the classical preservation property of ranking strategies $r(T(\KB))\nVdash \alpha \wedge t(e)\ptwiddle \bot$ and so again utilizing Proposition \ref{proposition:translation-of-PDSL-KB-into-Prop-KB}, $T^{-1}(r(T(\KB)))\nVdash \alpha\twiddle_e\bot$ or equivalently $T^{-1}(r(T(\KB)))\nVdash \Box_e \neg \alpha$. From a well-known result of modal semantics, it follows that $T^{-1}(r(T(\KB)))\Vdash \Diamond_e \alpha$ and so $\KB\dentails_r \Diamond_e\alpha$.
\end{proof}
\end{document}